\pdfoutput=1  

\documentclass[sigconf,authorversion,acmthm,screen]{acmart}

\usepackage[utf8]{inputenc}
\usepackage[noend, boxed, linesnumbered]{algorithm2e}
\usepackage{amsfonts}
\usepackage{amsmath}
\usepackage{comment}
\usepackage{listings}
\usepackage{multirow}

\usepackage{paralist}
\usepackage{tikz}
\usetikzlibrary{shapes,arrows}
\usepackage{url}
\usepackage{xspace}
\usepackage{wrapfig}

\begin{document}
\title{Verifying Controllers with Convolutional Neural Network-based Perception: A Case for  Intelligible, Safe, and Precise Abstractions}

%
%

\author{Chiao Hsieh}
\affiliation{%
    \institution{University of Illinois at Urbana-Champaign}
    \streetaddress{Coordinated Science Laboratory}
    \streetaddress{1308 W Main St}
    \city{Urbana}
    \state{IL}
    \country{USA}
    \postcode{61801}
}
\email{chsieh16@illinois.edu}

\author{Keyur Joshi}
\affiliation{%
    \institution{University of Illinois at Urbana-Champaign}
    \streetaddress{Coordinated Science Laboratory}
    \streetaddress{1308 W Main St}
    \city{Urbana}
    \state{IL}
    \country{USA}
    \postcode{61801}
}
\email{kpjoshi2@illinois.edu}

\author{Sasa Misailovic}
\affiliation{%
    \institution{University of Illinois at Urbana-Champaign}
    \streetaddress{Coordinated Science Laboratory}
    \streetaddress{1308 W Main St}
    \city{Urbana}
    \state{IL}
    \country{USA}
    \postcode{61801}
}
\email{misailo@illinois.edu}

\author{Sayan Mitra}
\affiliation{%
  \institution{University of Illinois at Urbana-Champaign}
  \streetaddress{Coordinated Science Laboratory}
  \streetaddress{1308 W Main St}
  \city{Urbana}
  \state{IL}
  \country{USA}
  \postcode{61801}
}
\email{mitras@illinois.edu}

\newcommand{\sayan}[1]{} 
\newcommand{\chiao}[1]{} 
\newcommand{\keyur}[1]{\textcolor{purple}{#1}}

\setlength{\abovecaptionskip}{2pt}
\setlength{\belowcaptionskip}{0pt}

\begin{abstract}
Convolutional Neural Networks (CNN) for object detection, lane detection, and segmentation now sit at the head of most autonomy pipelines, and yet, their safety analysis remains an important challenge. Formal analysis of perception models is fundamentally difficult because their correctness is hard if not impossible to specify.  We present a technique for inferring intelligible and safe abstractions for perception models from system-level safety requirements, data, and program analysis of the modules that are downstream from perception. The technique can help tradeoff  safety, size, and precision, in creating abstractions and the subsequent verification. We apply the method to two significant case studies based on high-fidelity simulations (a)~a vision-based lane keeping controller for an autonomous vehicle and (b)~a controller for an agricultural robot. We show how the generated abstractions can be composed with the downstream modules  and then the resulting abstract system can be verified using program analysis tools like CBMC.  Detailed evaluations of the impacts of size, safety requirements, and the environmental parameters (e.g., lighting, road surface, plant type)  on the precision of the generated abstractions suggest that the approach can help guide the search for corner cases and safe operating envelops.
\end{abstract}

\maketitle

\newcommand{\States}{\ensuremath{\mathcal{X}}\xspace}
\newcommand{\Envparams}{\ensuremath{E}\xspace}
\newcommand{\Unsafe}{\ensuremath{U}\xspace}
\newcommand{\Inv}{\ensuremath{\mathcal{I}}\xspace}

\newcommand{\Seterr}{\ensuremath{Z}\xspace}
\newcommand{\Setctrlout}{\ensuremath{\mathcal{U}}\xspace}
\newcommand{\Setimage}{\ensuremath{P}\xspace}

\renewcommand{\state}{\ensuremath{\mathbf{x}}\xspace}
\newcommand{\err}{\ensuremath{\mathbf{z}}\xspace}
\newcommand{\ctrlout}{\ensuremath{\mathbf{u}}\xspace}
\newcommand{\envparams}{\ensuremath{\mathbf{e}}\xspace}
\newcommand{\image}{\ensuremath{\mathbf{p}}\xspace}

\newcommand{\camera}{\ensuremath{\mathit{sensor}}\xspace}
\newcommand{\nnet}{\ensuremath{\mathit{DNN}}\xspace}
\newcommand{\controller}{\ensuremath{\mathit{control}}\xspace}
\newcommand{\dynamics}{\ensuremath{\mathit{dynamics}}\xspace}

\newcommand{\Fcam}{\ensuremath{\mathit{s}}\xspace}
\newcommand{\Fnn}{\ensuremath{\mathit{h}}\xspace}
\newcommand{\Fctrl}{\ensuremath{\mathit{g}}\xspace}
\newcommand{\Fdyn}{\ensuremath{\mathit{f}}\xspace}
\newcommand{\Fmm}{\ensuremath{\mathit{m}^*}\xspace}
\newcommand{\Fapprox}{\ensuremath{M}\xspace}
\newcommand{\Ftol}{\ensuremath{\mathcal{R}}\xspace}

\newcommand{\vf}{\ensuremath{v_f}\xspace}
\newcommand{\dT}{\ensuremath{\Delta T}\xspace}

\newcommand{\Ai}{\ensuremath{\mathbf{A}_i}\xspace}
\newcommand{\bi}{\ensuremath{\mathbf{b}_i}\xspace}
\newcommand{\ri}{\ensuremath{r_i}\xspace}

\newcommand{\Sys}{\ensuremath{\mathit{Sys}}\xspace}
\newcommand{\IdealSys}{\ensuremath{\mathit{Sys}^*}\xspace}
\newcommand{\ApproxSys}{\ensuremath{\widehat{\mathit{Sys}}}\xspace}
\newcommand{\lyapunov}{\ensuremath{V}\xspace}

\section{Introduction}
\label{sec:intro}

Assuring safety of autonomous systems is a looming technical challenge.
Regulatory agencies from many industries -- aerospace~\cite{AI-EASA21}, automotive~\cite{koopman2019safety}, trucking, agriculture, manufacturing -- are busily drawing-up definitions, concepts of operations, processes, and guidelines for enforcing safety.
Deep learning, a key enabler for autonomy, is becoming indispensable for perception, and increasingly central for decision and control.
Not coincidentally, research on formally verifying isolated neural networks (NN), has received a degree of attention~\cite{reluplex,NeVer,NNVTool,sherlock,verisig,bak2021second}.
\emph{System-level safety assurance} brings a somewhat different set of opportunities and challenges and has been addressed in a small number of recent efforts~\cite{verifAI,katz2021verification,dean2020a}.

Even though writing formal requirements for CNN perception models may be ill-posed~\cite{seshia2018formal}, safety requirements for autonomous systems using those very models, is usually fairly obvious.
A ``lane'' may be difficult to specify in terms of pixel intensity thresholds, but the safety requirements of a lane keeping control system are less mysterious.
It is well-known that CNNs have fragile decision boundaries and are susceptible to adversarial inputs~\cite{szegedy2014intriguing}.
Since the  existing NN verification tools verify properties around a  small neighborhood of the input space of the NN, presence of adversarial inputs make the NN verification results conservative. System-level safety analysis, on the other hand, has to deal  with the temporal evolution of the whole system (including the NNs), and therefore, could benefit from the smoothness of the physical world, and be more robust and precise by overlooking the occasional mis-classifications.  Such robustness has indeed been empirically observed~\cite{lu2017stopsign}.

In this paper, we propose a safety assurance technique of control systems that use CNN models for perception.
Since completely formal specification and verification of such CNN models is hard, if not impossible, our method creates \emph{abstractions} of the CNN-based perception subsystem.
An abstraction \Fapprox (or over-approximation)
of a CNN perception subsystem $P$ is obviously useful for safety analysis---if the abstract system obtained by replacing $P$ with \Fapprox can be verified safe, then we can infer safety of the original  concrete system.
However, there are two barriers to this approach.

First, the problem of verification of the CNN $P$ is now shifted to the problem of proving that \Fapprox is an abstraction of $P$. For the same reasons mentioned above, we will not attempt to solve this problem formally.  Instead, we aim to provide statistical evidence about the {\em precision} of the abstraction.

Second, the abstraction $M$ should not only prove safety of the system, but also it should be \emph{intelligible}. Safety assurances should not only come from tests, proofs, processes, and verification artifacts showing \emph{that the system is correct}, but also from explications on \emph{why it is so}~\cite{adadi2018xai, baier2021XAI}.
We agree with this sentiment and will aim to create abstractions that are intelligible by human designers, testers, and auditors.

These three axes---safety, intelligibility, and precision---define a space for exploring different safety assurance methodologies for  autonomous systems.
In this paper, we present a particular method that constructs \emph{property-guided, piece-wise affine abstractions}.
To our knowledge, this is the first abstraction-based approach to verify control systems that use CNNs for perception.
We use a piece-wise affine template for \Fapprox:
Suppose the ground truth perception input to the control system is $\Fmm(\state, \envparams)$,
for a given state $\state$ and a set of environmental parameters $\envparams$.
These parameters could include lighting conditions, road surface, weather conditions, etc.
Then $\Fapprox(\state,\envparams)$ will be a \emph{set-valued function},
where the center (mean or bias) of the each set is a piece-wise affine function $\Ai(\Fmm(\state,\envparams)) + \bi$ of the ground truth $\Fmm(\state,\envparams)$.
We need not know this ground truth function \Fmm or its precise dependence on the environmental parameters, however, we can infer the linear model using  regression on the  data generated by running the CNN model $P$ with different \state and \envparams inputs. In the case of synthetic data generated using a simulator, as we do in our experiments, we can also label the $(\state,\envparams)$ data with  the corresponding ground truth value. This improves the precision of \Fapprox, by indirectly reducing the error with respect to the quantity that the CNN-based perception subsystem  $P$ is trained to estimate, namely, the ground truth.

While the center (mean) of the set $\Fmm(\state,\envparams)$ is defined by training data, the size and shape of the set (variance) is inferred from safety. Assume that the control system is safe with respect to a given unsafe set $\Unsafe$, when it uses  perfect perception.
Using program analysis tools like CBMC~\cite{CBMC} and IKOS~\cite{IKOS} on the code for the controller,
we infer the set of \emph{unsafe} outputs from \Fapprox for any $\state$.
Then, the set-valued output from $\Fapprox(\state,\envparams)$ is determined to be the \emph{largest set},
centered at  $\Ai(\Fmm(\state,\envparams)) + \bi$, that keeps the system safe.
The computation of this largest set is an optimization problem.

Our method produces intelligible abstractions. The resulting output abstraction \Fapprox is a piece-wise affine set-valued function of the actual variable that the perception system $P$ is trying to estimate.
The abstraction \Fapprox is guaranteed to be safe relative to the given property \Unsafe.
We check this using CBMC by plugging-in \Fapprox into the downstream modules of the  control system.
This safety-first approach can be precision-challenged under some conditions.
In our experiments with two end-to-end autonomous systems---vision-based lane keeping for an electric vehicle and a vision-based corn row scouting robot---
we evaluated the abstractions from large variations of environments such as roads with varying numbers of lanes, lighting conditions, different types of crops and fields,
generated with high fidelity simulator Gazebo.
We observe that  in certain parts of the state space,
the computed safe abstractions are not able to match the original perception subsystem $P$ very accurately due to the strong inductive invariant that we used as the requirement.
While for some other conditions, the precision of the abstraction can be over 90\% match (explained in detail in Section~\ref{subsec:precision}) with finer abstractions, narrower environmental variations, and alternative inductive requirements.

We believe that the trifecta of safety, intelligibility, and precision provides a useful framework for constructing abstraction and verifying  autonomous systems. In addition to providing assurances for safety requirements, the notion of precision of abstractions developed here can shed light on parts of the state space and environment where the CNN-based perception system is  fragile, and likely to violate requirements. These quantitative insights can  design  perception and control subsystems, as well as inform the definition of the system-level {\em operating design domains (ODDs)}~\cite{koopman2019many}.

In summary, our contributions of this work are:
\begin{itemize}
\item Formalization of an intelligible, safe, and precise abstractions for CNNs as a piece-wise affine set-valued function \Fapprox.
\item Approach to find \Fapprox through a combination of linear regression and constrained optimization.
\item Demonstration that \Fapprox can be composed with the downstream modules and verified using existing tools.
\item Quantitative evaluation of the precision of \Fapprox with respect to the original CNN based perception subsystem  $P$, and the interpretation of its dependence on various factors.
\end{itemize}

\section{Related works}
\paragraph{XAI}
The explainable AI (XAI) and interpretable machine learning (IML) research areas have been explosively growing over the past five years.
Figure~1 of \cite{IML-brief-history21} suggests that in 2020 there were 400+ publications related to interpretability alone. The survey articles~\cite{adadi2018xai,bodria2021xai-benchmarking,IML-brief-history21}
provide systematic overview of the terminologies and the available techniques for different types of AI models for  text, image, and tables.
Some of the prominent techniques rely on the notions of feature importance~\cite{apley2019visualizing},
Shapley values~\cite{janzing20a},
and counter-factual explanations~\cite{LaugelLMRD19}. Our piece-wise affine abstractions for CNNs is a natural interpretable model, but we have not seen this used in the literature so far.

Using the terminology of the XAI literature, our method provides \emph{model-agnostic}, \emph{global} interpretations of image-based AI models.
The notion of \emph{interpretability} itself has differing  definitions in the literature.
Our interpretations help with \emph{transparency}, or equivalently \emph{intelligibility}, in that, they help a human to understand the functioning of the black-box CNN.
Our interpretations and methods are model agnostic because they do not rely on the internal structure or workings of the CNN models, instead, they only require input-output data.
Also, out interpretations are \emph{global} in the sense that they provide interpretations for the entire domain.





\paragraph{Analysis of closed loop systems with NNs}
The closest related works are VerifAI~\cite{verifAI} and a recent work \cite{katz2021verification}.
VerifAI~\cite{verifAI} and related publications provide a comprehensive framework to analyze a closed loop system with ML-based perception components.
They focus on the falsification of the system specification, data augmentation, and redesigning the neural network.
Their techniques include fuzz testing, simulation, counterexample guided data augmentation, syntheses of hyper-parameters and model parameters.
Our work provides a safe abstraction, and therefore complements the falsification approaches of VerifAI.

Our work is similar in spirit to  the white paper~\cite{pasareanu2018compositional} and the work reported in~\cite{katz2021verification}\footnote{Unpublished at the time of writing.}, in that, they propose using abstraction/contracts for perception components.
In~\cite{katz2021verification}, the authors train generative adversarial networks (GANs) to produce a simpler neural network that abstracts away image sensing and image based perception.
The simpler network directly transforms states and environment parameters to estimates similar to our abstraction \Fapprox. In comparison, our work provides an intelligible set-valued function instead of a simpler network.

In addition, \cite{dean2020a} considers synthesizing robust perception based controller. Plenty of recent works focus on verification~\cite{NNVTool,ivanov2020verify},
reachability analysis~\cite{dutta2019reachability,fan2020reachnn,verisig,hu2020reachsdcp,everett2021reachability},
statistical model checking~\cite{DSMC}, and synthesis~\cite{ivanov2021compositional} on \emph{neural feedback systems} with neural network controllers.
\cite{Julian2019GuaranteeingSF,julian2021acasxu} specifically focus on developing and verifying a neural network replacement for ACAS-Xu collision avoidance decision tables.
Such controller NNs are typically  very small compared to perception CNNs.
Our work is the \textbf{first} to provide intelligible abstractions for CNN-based perception and provides safety guarantees for the closed loop system.

\paragraph{Isolated neural network verification}
\cite{muvva2021uavlec} evaluates the safety assurance of ML-based Perception and Control in UAS via simulation
and discusses the challenges of integrating existing neural network verification tools for the system level analysis.
The authors point out that the neural network used for perception in small UAS is larger than any previously analyzed NN in the competition~\cite{bak2021second}.
GAS~\cite{GAS} analyzes the impact of NN perception uncertainty on vehicles using an approximated perception model and Generalized Polynomial Chaos. Unlike our approach, GAS only estimates the probability that the vehicle will reach an unsafe state. Its perception model approximation does not incorporate system-level specifications.

\section{System-level safety assurance}
\label{sec:prelim}
The problem of assuring safety of a control system can be stated as follows: given a control system or a program \Sys on a state space \States,
we would like to check that it satisfies an invariant $\Inv\subseteq\States$.
For example, for a lane keeping control system for a vehicle, the invariant requirement is that the vehicle always remains within the lanes.
This seemingly textbook statement of the problem is complicated by two factors in an actual autonomous system.
\begin{figure}
    \centering
    \includegraphics[width=0.3\linewidth,clip,trim=80cm 20cm 1cm 20cm]{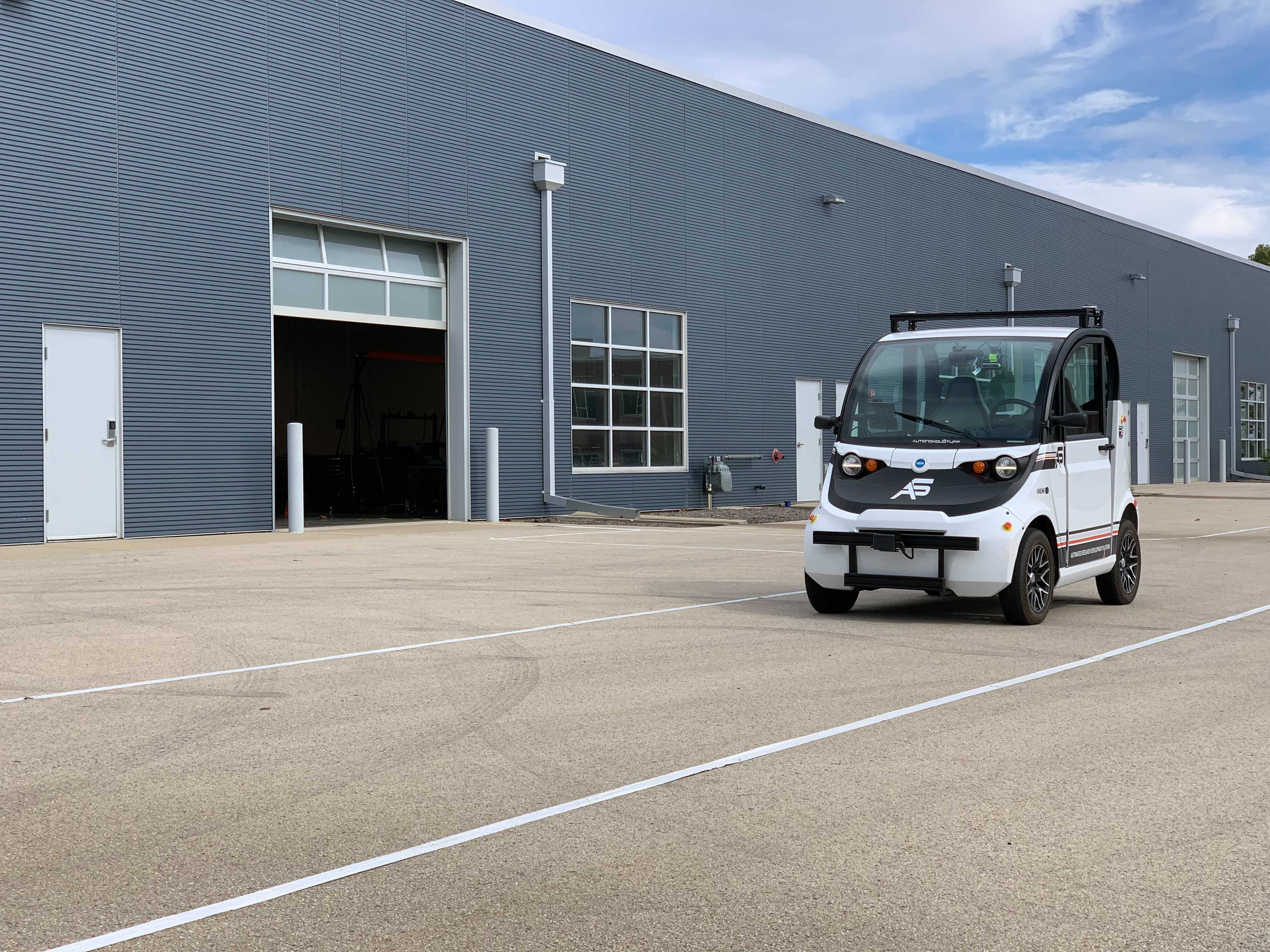}
    \includegraphics[width=0.24\linewidth,clip,trim=1cm 0cm 3cm 4.2cm]{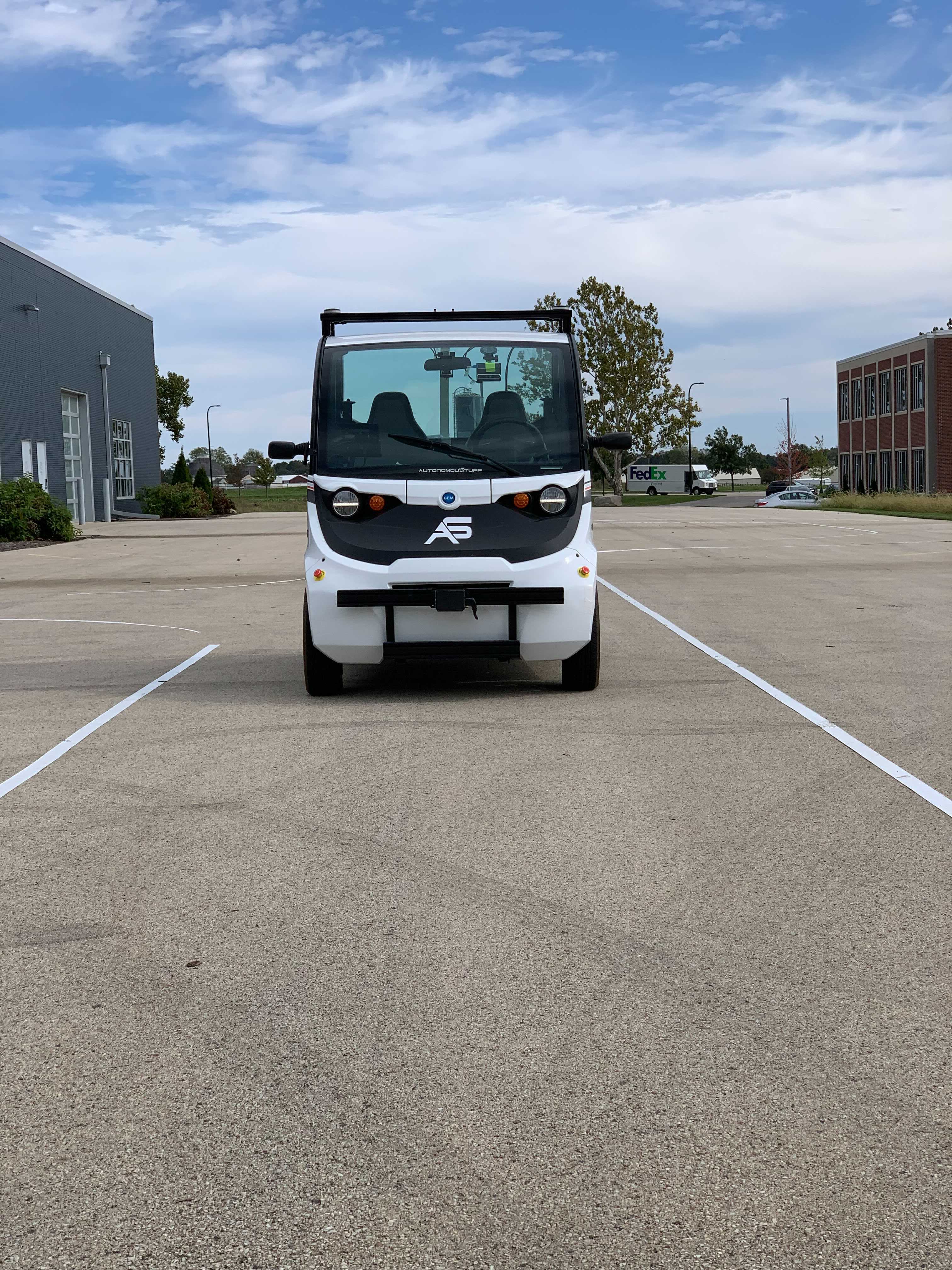}
    \includegraphics[width=0.42\linewidth,trim=20cm 0cm 15cm 15cm,clip]{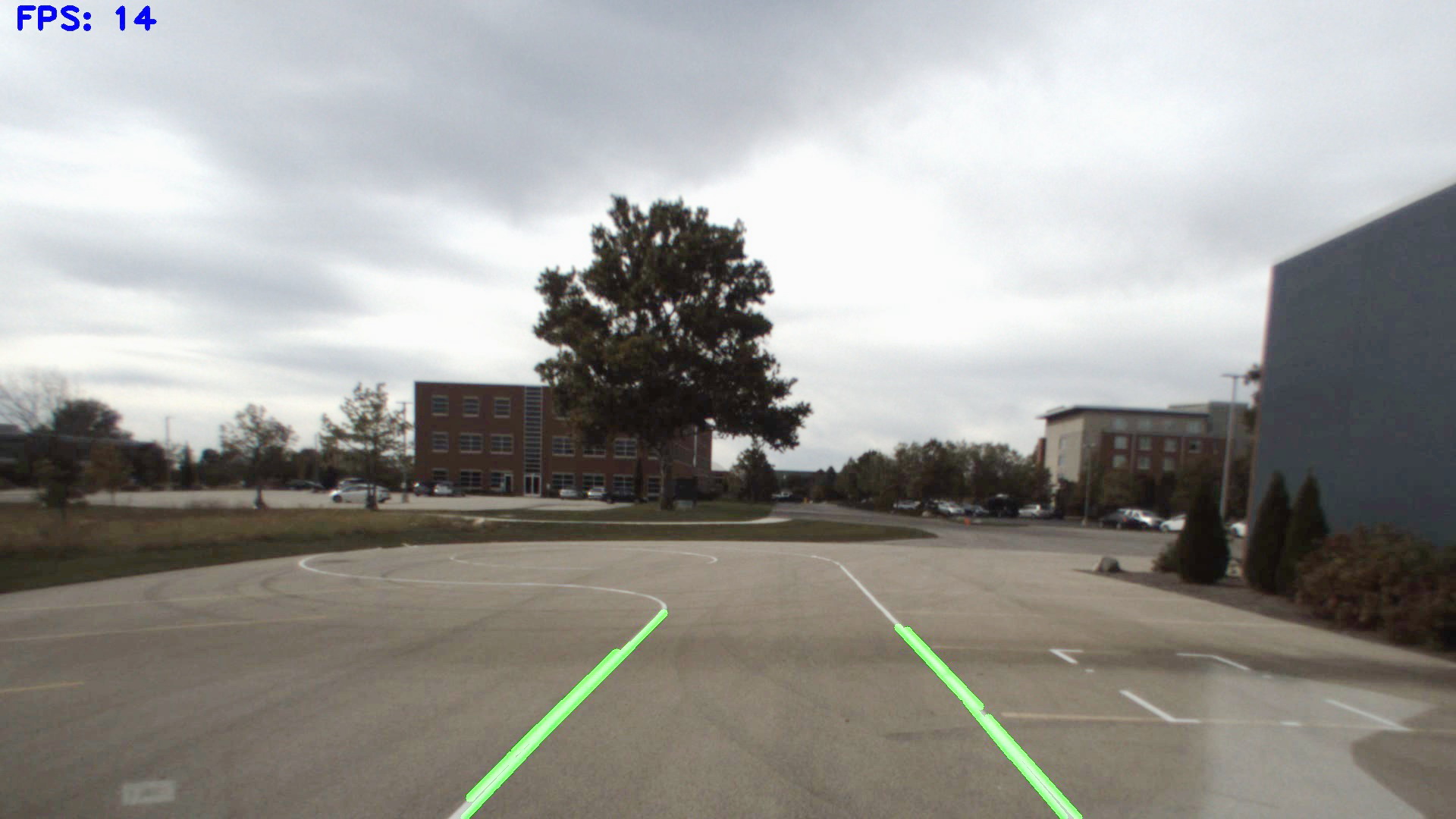}
    \caption{Vision-based lane keeping control on AV platform.}
    \label{fig:GEM}
\end{figure}

\begin{figure}[ht]
    \centering
    \includegraphics[width=\linewidth,clip,trim=0cm 0cm 0cm 0cm]{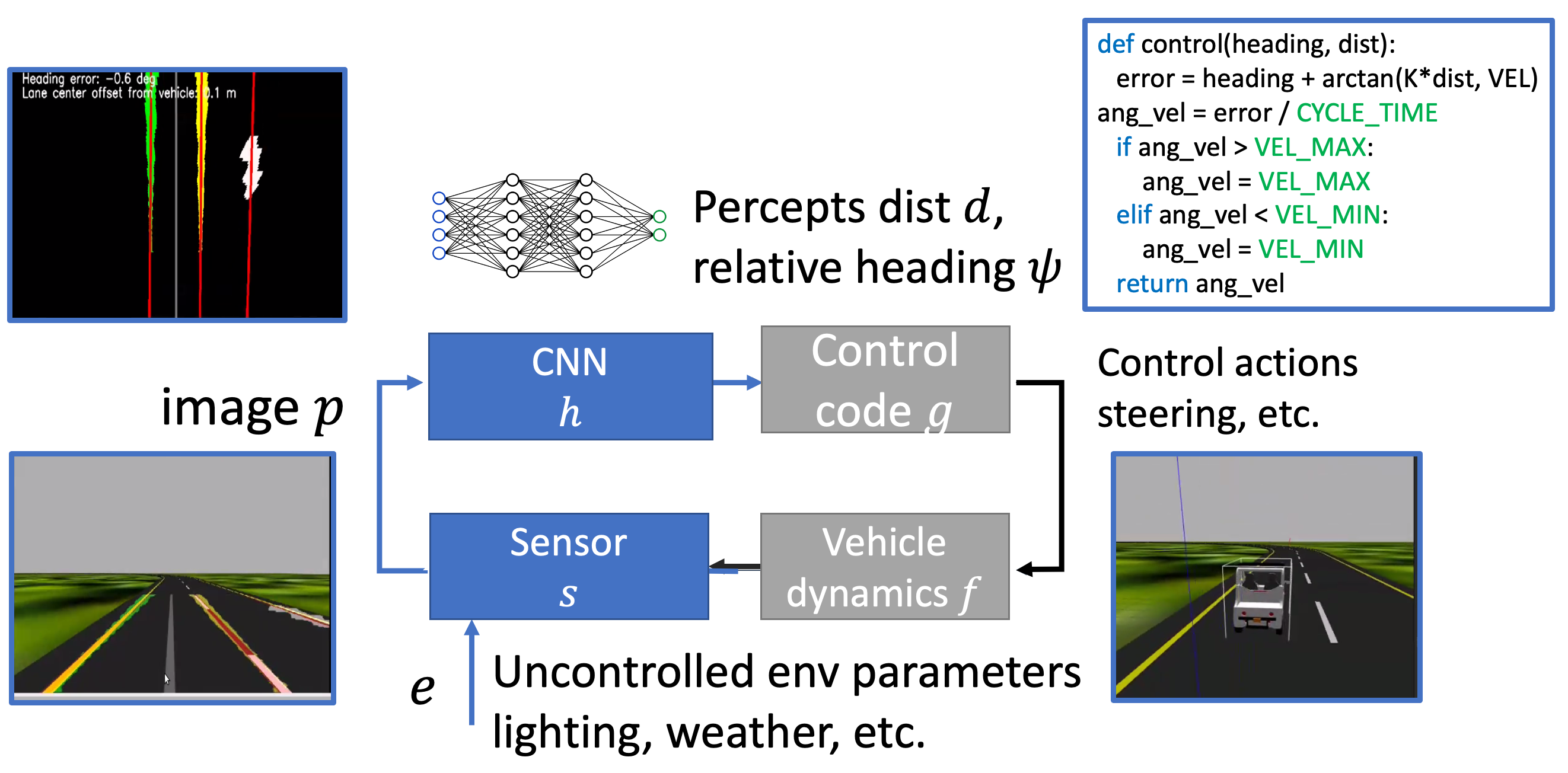}
    \caption{Closed-loop control system \Sys with camera and CNN-based perception.}
    \label{fig:sys}
\end{figure}

First, \Sys uses deep neural networks (DNNs) for perception -- converting pixels to \emph{percepts} such as deviation from lane-center,
and such perception systems are not amenable to formal specification and verification.
Secondly, the DNN's output depends on environmental factors \Envparams such as lighting, texture, and pavement moisture.
These dependencies are neither well-understood nor controllable.

\subsection{System description}
\label{sec:system}
Before stating our safety analysis problem more precisely, we set up a mathematical representation of \Sys.
We view \Sys as a discrete time transition system  with four interconnected components transforming different types of data and ultimately defining the state transitions (Figure~\ref{fig:sys}).

The \nnet is a deep neural network model for perception. It takes an image (or a high-dimensional vector) \image as an input and produces a percept or a low-dimensional \emph{estimate} vector  $\err = \Fnn(\image) $ as the output. In a lane tracking system, this percept $\err$ could be, for example, the position of the camera relative to the lanes seen in the image.
That is, we model the \nnet as a function $\Fnn: \Setimage \rightarrow \Seterr$ mapping the space of images $\Setimage$ to the space of percepts \Seterr.

The \controller module is a program that takes a percept \err as an input and produces a  control action $\ctrlout=\Fctrl(\err)$ as the output.
In a lane tracking system, the control action $\ctrlout$  is a vector of throttle, steering, and brake signals.
The implementation of the controller \controller may involve a number of modules including navigation, planning, and optimization.
For simplicity, we model the \controller as a function $\Fctrl:\Seterr \rightarrow \Setctrlout$ mapping the space of precepts to the space of control actions.

Then \dynamics defines the evolution of the system state $\state$ as a function of the previous state and the output from the \controller.
We model the \dynamics as a function $\Fdyn:\States \times \Setctrlout \rightarrow \States$.
In our example, the state \state of the vehicle includes its position, orientation, velocity, etc.,
and the dynamics function defines how the state changes with a given control action $\ctrlout \in \Setctrlout$.
In this paper, we consider discrete time models, and write the state at time $t+1$ as
\[
\state_{t+1} = \Fdyn(\state_t, \ctrlout_t),
\]
where $\state_t$ and $\ctrlout_t$ are the state and the control actions at time $t$.
This state transition function could be generalized to a relation to accommodate uncertainty, without significantly affecting our framework or the results.

The final component closing the loop is the \camera which defines the image \image as a function of the current state \state and a set of  non-time varying, \emph{environmental parameters} \envparams.
In our example, these parameters include, for example, lighting conditions, nature of the road surface, types of markings defining lanes, etc.
We model the \camera as a function $\Fcam: \States \times \Envparams \rightarrow \Setimage$, where $\Envparams$ is the space of environmental parameter values.
In a real system, we may not know all the environmental parameters, they may not be time-invariant, and their precise functional influence on the image will also be unknown.
Therefore, it does not make sense to prove anything mathematically about \Fcam.
For the purpose of generating abstractions of $\Fnn\circ\Fcam$, we reasonably assume that we can sample inputs of $\Fcam$ according to some distribution over $\Envparams$ and $\States$. In our experiments, we generate synthetic data using a simulator, and the same could also be done with the actual vehicle platform at a higher cost.

\subsection{Assurances for closed-loop system}
The behaviors of the overall system are modeled as sequences of states called {\em executions\/}. Given an initial state $\state_0 \in \States$ and an environmental parameter value $\envparams \in \Envparams$, an execution of the overall system $\alpha(\state_0,\envparams)$ is a sequence of states $\state_0, \state_1, \state_2, \ldots$ such that for each index $t$ in the sequence:
\begin{eqnarray}
\state_{t+1} = \Fdyn(\state_t, \Fctrl(\Fnn(\Fcam(\state_t,\envparams)))).
\label{eq:closed-system}
\end{eqnarray}
In an ideal world, we would like to have methods  that can assure that, given a range of environmental parameter values $\Envparams_0 \subseteq \Envparams$, an unsafe set $U \subseteq \States$,  and a set of initial conditions $\States_0 \subseteq \States$, none of the resulting executions of the system from $\States_0$ can reach $U$ under any choice of $\Envparams_0 $. Such a method will be a useful tool in checking safety of autonomous systems. Also, in indirectly helping determine   $\Envparams_0$ for which the system can be assured to be safe, such methods can be used as a scientific basis for specifying the {\em operating design domain (ODD)}~\cite{koopman2019safety} for the control system.

Since, the functions  $\Fcam$ and $\Fnn$ are partially unknown with unknown dependence on $\envparams$ and $\state$, it is unreasonable to look for the above type of methods. Instead, in this paper, we develop a method for the following weaker problem:
\paragraph{{\bf Problem.}}
Given an unsafe set $\Unsafe \subseteq \States$ and a  range of environmental parameters $\Envparams_0 \subseteq \Envparams$, find an {\em abstraction\/} $\Fapprox_U$ of the perception system $\Fnn\circ\Fcam$, such that it is:
\begin{enumerate}[(a)]
\item {\em Safe\/}, i.e., $\Fapprox_U$  used in the  closed  loop  system substituting $\Fnn\circ\Fcam$ makes the resulting system safe with respect to $\Unsafe$.
\item {\em Intelligible\/}, i.e., human designers can understand the  behavior of $\Fapprox_U$.
 \item {\em Precise\/}, that is, $\Fapprox_U$ and $\Fnn\circ\Fcam$ are close.
\end{enumerate}

The goal of this paper is to explore strategies for creating such abstractions.
%
Note that the  construction of $\Fapprox_U$ will rely on both the knowledge of the unsafe set $U$ and
the range of environmental parameters $\Envparams$ under consideration.

For the substitution to make sense, the  abstraction $\Fapprox$ must have the same  type as   $\Fnn\circ\Fcam$, however, to allow it higher precision across different environments, we make it a set valued function. That is, $\Fapprox_U: \States\times\Envparams \rightarrow 2^\Seterr$.

Since, the actual perception system $\Fnn\circ\Fcam$ and its dependence on the environment $\Envparams$ is incompletely understood, any assertion about the closeness to the abstraction $\Fapprox_U$ will have to be empirically evaluated.
There are many options for measuring closeness that can factor in information about the environmental parameters.
We will see later that indeed fine-grained measurement of closeness is possible.
We will also discuss how such comparisons can guide both the process of data collection for more precise empirical evaluations,
as well as the inference of operating design domains for the system.

\subsection{An example: Vision-based lane keeping}
\label{sec:ex1}

As a example, we consider a computer vision-based lane keeping control system as shown in Figures~\ref{fig:GEM} and~\ref{fig:sys}.

\paragraph{Dynamics and control.}
The vehicle state $\state \in \States$ consists of the 2D position $(x, y)$ of the center of the front axle in a global coordinate system, and the heading angle $\theta$ w.r.t the $x$-axis.
The input $\ctrlout \in \Setctrlout$ is the steering angle $\delta$.
The discrete time model for the above state vector is the well-known kinematic bicycle function~\cite{frazzoli_survey_2016} $\Fdyn(\state, \ctrlout)$:
\begin{align*}
x_{t+1} &= x_t + \vf\cdot\cos(\theta_t + \delta)\cdot\dT \\
y_{t+1} &= y_t + \vf\cdot\sin(\theta_t + \delta)\cdot\dT \\
\theta_{t+1} &= \theta_t + \vf\cdot\frac{\sin(\delta)}{L}\cdot\dT
\end{align*}
where $\vf$ is the  forward velocity, $L$ is the wheel base, and $\dT$ is a time discretization parameter.
We discuss the impact of different vehicle models on our methodology in Section~\ref{sec:discussion}.

\begin{wrapfigure}{r}{0.22\textwidth}
    \centering
    \includegraphics[width=0.22\textwidth]{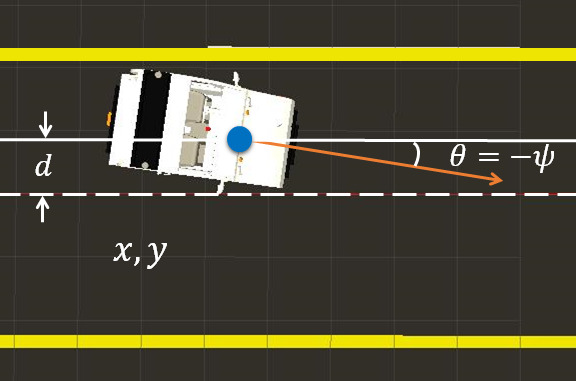}
    \caption{\small State  $(x,y,\theta)$ and perception variables $(d,\psi)$ for lane keeping.}
    \label{fig:lane-err}
\end{wrapfigure}
The input to the dynamic function $\Fdyn$ comes from the decision and control program. Here we use the standard Stanley controller~\cite{hoffmann_stanley_2007} used for lateral control of vehicles. This controller uses the percept $\err \in \Seterr$, which consists of the heading difference $\psi$ and cross track distance $d$ from the center of the lane to the ego-vehicle.
In Figure~\ref{fig:lane-err} the heading $\theta$ coincides with the negation of heading difference $-\psi$,
but this happens only in the special case where the lane is aligned with the $x$-axis.

The controller function $\Fctrl(\err)$ is defined as:
\[
\delta = \Fctrl(d, \psi) = \begin{cases}
\psi + \arctan\left(\frac{K\cdot d}{\vf}\right), & \text{if } \left|\psi + \arctan\left(\frac{K\cdot d}{\vf}\right)\right| < \delta_{max} \\
 \delta_{max}, & \text{if } \psi + \arctan\left(\frac{K\cdot d}{\vf}\right) \geq \delta_{max} \\
-\delta_{max}, & \text{if } \psi + \arctan\left(\frac{K\cdot d}{\vf}\right) \leq -\delta_{max} \\
\end{cases}
\]
where $\delta_{max}$ is the steering angle limit and $K$ is a controller gain parameter.

\paragraph{Perception}
Now we describe the complex perceptual part of the system that estimates heading difference $\psi$ and cross track distance $d$ using computer vision.
First, the sensor function \Fcam uses cameras to capture an image, and processes the image through a sequence of computer vision algorithms including cropping the region of interest,
undistortion, morphological transformations, resizing, etc., to prepare the input image \image for the DNN. The particular DNN used here is
LaneNet~\cite{nevan_lanenet_2018} which uses  $512\times256$ RGB images to detect lane pixels.
Internally, LaneNet contains two sub-nets for both the identification and instance segmentation of lane marking pixels;
then a curve fitting algorithm is applied on identified lane marking pixels to represent each detected lane as a polynomial function.
Further, the perspective warping is applied to map the polynomial function to bird's eye view,
which gives the final percept $\err = (d, \psi)$ of the relative position of the vehicle to the lane center ($d$) and the heading difference ($\psi$), as shown in Figure~\ref{fig:sys}.

\paragraph{System safety requirement.}
A common specification for lane keeping control is to avoid going out of the lane boundaries.
We assume that the vehicle is driving on a straight road with lane width $W$. For the purpose of simplifying exposition, we assume that the center line is aligned with the $x$-axis of the global coordinate system. Thus, the unsafe set can be specified as:
$$
\Unsafe = \left\{(x,y,\theta) \mid |y| > \frac{W}{2}\right\}.
$$

\section{Safe Intelligible Abstractions}
\label{sec:synthesis}

In this section we will discuss our method for constructing the  the abstraction $\Fapprox$ for the perception system $\Fnn\circ\Fcam$.  $\Fapprox$ will be a piece-wise affine set-valued function of the ground truth value that $\Fnn\circ\Fcam$ is supposed to estimate.
Section~\ref{sec:abs-system} sets the stage showing how a set-valued abstraction of $\Fnn\circ\Fcam$ defines an abstraction of the original control system \Sys, and therefore, can be useful for verification. Section~\ref{sec:cons-sia} presents the main algorithm that learns, from perception data, the center (mean) of the output set $\Fapprox(\state,\envparams)$. Section~\ref{subsec:safe-precise} defines the next step in the construction of $\Fapprox$. This step analyzes  the control program $\Fdyn \circ \Fctrl$ to optimize the shape and the size of the output set around the mean, to assure  the safety of the abstract system with $\Fapprox$ with respect to the unsafe set $\Unsafe$. 
Section~\ref{sec:soundness} establishes the safety of the constructed abstraction, not only at the theoretical model level, but it also shows how $\Fapprox$ can be plugged in to the rest of the $\Sys$ code and verified using program analysis tools, namely CBMC~\cite{} in our work. Finally,  Section~\ref{subsec:precision} discusses our methods for empirically evaluating the  precision of $\Fapprox$.

\subsection{Abstract perception in closed-loop}
\label{sec:abs-system}

We will construct a set-valued perception function  $\Fapprox$ that abstracts the complex perception system $\Fnn\circ\Fcam$ and meets the three requirements of safety, intelligibility, and precision.
For the safety requirement, our constructed function $\Fapprox: \States\times\Envparams \rightarrow 2^\Seterr$ should be such that when it is ``substituted'' in the closed loop system of Equation~(\ref{eq:closed-system}), the resulting system is safe with respect to the requirement $\Unsafe$.
Formally, substituting $\Fnn\circ\Fcam(\state, \envparams)$ with $\Fapprox(\state, \envparams)$,
the result is the non-deterministic system $\ApproxSys(\Fapprox)$ given by:
\begin{eqnarray}
\state_{t+1} \in \{\Fdyn(\state_t, \Fctrl(\err)) \mid \exists \err \in \Fapprox(\state_t, \envparams) \}.
\label{eq:abs-sys}
\end{eqnarray}
That is, when the actual system state is $\state_t$ (and the environmental parameters $\envparams$), then the output from the abstract perception function $\Fapprox$ can be {\em anything\/} in the set $\Fapprox(\state_t, \envparams)$. This set-valued approach is a standard way for modeling noisy sensors.

\begin{definition}
\label{def:overapprox}
A function $\Fapprox: \States\times\Envparams \rightarrow 2^\Seterr$ is an \emph{abstraction} of $\Fnn\circ\Fcam$ if:
\begin{eqnarray}
\forall \envparams\in\Envparams, \forall \state \in \States. \Fnn\circ\Fcam(\state, \envparams) \in \Fapprox(\state, \envparams).
\label{eq:perc-abs}
\end{eqnarray}
\end{definition}
This definition ensures $\Fapprox(\state, \envparams)$ covers all possible percepts  from $\Fnn\circ\Fcam(\state, \envparams)$ for all states and environments.
If a function $\Fapprox$ is an abstraction of $\Fnn\circ\Fcam$, then it follows that $\ApproxSys(\Fapprox)$ is an abstraction of \Sys, that is, the set of executions of $\ApproxSys(\Fapprox)$ contains the executions of \Sys. Therefore, any state invariant $\Inv \subseteq \States$ for $\ApproxSys(\Fapprox)$ carries over as an invariant of $\Sys$.

\begin{proposition}
\label{prop:abs}
If  $\Fapprox$ is an abstraction of $\Fnn\circ\Fcam$ then $\ApproxSys(\Fapprox)$ is an abstraction of $\Sys$.
\end{proposition}
Fixing an arbitrary initial state $\state_0$ and an environment $\envparams$, this follows immediately from Equation~(\ref{eq:perc-abs}) by deriving:
\[
   \Fdyn(\state, \Fctrl(\Fnn\circ\Fcam(\state, \envparams))) \in \{\Fdyn(\state, \Fctrl(\err)) \mid \exists \err \in \Fapprox(\state, \envparams) \}.
\]

Definition~\ref{def:overapprox} is too general to be useful for constructing safe, intelligible, and precise abstractions. At one extreme, it allows the definition $\Fapprox(\state,\envparams) := \{\Fnn\circ\Fcam(\state,\envparams)\}$ which is exactly same as the original perception system, but does not help with  intelligibility nor with safety.
At the other end,
we can make $\Fapprox(\state,\envparams)$ to be all of $ \Seterr$, which may be intelligible but not useful for safety.

Our approach is to \emph{utilize available information about safety of the control system without perception}.
Informally, consider a version of the closed loop control system that  uses the ground truth values of $\psi,d$ instead of relying on camera and DNNs to estimate these values.
In order to prove safety of this {\em ideal system\/} with respect to $\Unsafe$, we can use standard invariant assertions~\cite{prajana2004barrier, sankaranarayanan2004invariant, platzer2008logic, platzer2010difflogic,MitraCPSBook2021}.
We will construct the abstraction \Fapprox for \Sys in a way that can utilize the knowledge such invariants.

\begin{definition}
\label{def:safety-dir}
Given a set $\Inv \subseteq \States$ and an abstract perception function  $\Fapprox: \States\times\Envparams \rightarrow 2^\Seterr$,
\Fapprox preserves \Inv if
\[
\forall \envparams\in\Envparams.\forall \state \in \Inv, \forall \err \in \Fapprox(\state, \envparams), \Fdyn(\state,\Fctrl(\err)) \in \Inv.
\]
\end{definition}
Finding an invariant preserving abstraction satisfying Definition~\ref{def:safety-dir} will guide us towards creating more practical abstractions of the perception system.

\subsection{Learning piece-wise abstractions from data}
\label{sec:cons-sia}

\SetKwProg{prog}{Function}{}{}
\SetKwFunction{Fmain}{ComputeAbstraction}
\SetKwFunction{Fregression}{LinearRegression}
\SetKwFunction{Fmindist}{MinDist}

For the abstract perception function $\Fapprox:\States \times \Envparams \rightarrow 2^\Seterr$ to be intelligible, for any $\state \in \States$ and $\envparams \in \Envparams,$ the output $\Fapprox(\state,\envparams)$ should be related to the ground truth value $\err^* \in \Seterr$ that the perception system is supposed to estimate.
For example, for a given state $\state = (x,y,\theta)$ of the vehicle in the lane keeping system, and a given configuration of the lanes defined by $\envparams$, the ground truth $\err^* = (d^*, \psi^*)$---consisting of the relative position to lane center ($d^*$) and the angle with the lane orientation ($\psi^*$)---is uniquely determined by the geometry of the vehicle, the camera, and the lanes.
The perception system $\Fcam \circ \Fnn$ is designed to capture this functional relationship between $\state$ and $\err$ (and it is affected by the environment $\envparams$).
For the sake of this discussion, let $\Fmm(\state) = \err^*$ be the idealized function that gives the ground truth percept $\err^*$ for any state $\state$.
We may not know $\Fmm$ and only have access to it as a black-box function.
Nevertheless, a well-trained and well-designed perception system $\Fnn\circ\Fcam$  should minimize the error\footnote{The precise choice of the error function is a design parameter and we will discuss this further in later sections.}  $||\Fmm(\state) - \Fnn\circ\Fcam(\state,\envparams)||$ over  relevant states and environmental conditions. As $\Fapprox$ is  an abstraction of $\Fnn\circ\Fcam$, therefore, to achieve precision, $\Fapprox$ should also minimize error with respect to $\Fmm(\state)$.

In this paper, we consider a  piece-wise affine structure of the abstraction $\Fapprox$.
This is an expressive class of functions with conceptual and representational simplicity, and hence human readable and comprehensible.
First, given a partition  $\{\States_i\}_{i=1\dots N}$ of the target invariant domain,
i.e., $\Inv= \bigcup_{i=1}^N \States_i$,
we define \Fapprox as:
\[
\Fapprox(\state, \envparams) = \begin{cases}
\Ftol_1(\Fmm(\state)), &\text{iff} \ \state\in\States_1 \\
\qquad \vdots \\
\Ftol_N(\Fmm(\state)), &\text{iff}\ \state\in\States_N
\end{cases}
\]
where we search for
$\Ftol_i:\Seterr\rightarrow 2^\Seterr$ that returns a neighborhood around $\Fmm(\state)$.

In what follows, we will show how $\Ftol_i$'s can be derived as a linear function of $\Fmm(\state)$ that is both safe with respect to the target invariant $\Inv$ and minimizes error with respect to training data samples available from the perception system.
\Fmain gives our algorithm for computing this abstraction for each partition $\States_i$.

\begin{algorithm}[h]
\KwIn{Subspace $\States_i$; Invariant \Inv; Dynamics \Fdyn; Control \Fctrl; Perfect Estimation \Fmm}
\KwData{Training set of ground truth vs perceived percepts $L=\{(\err^*_1, \err_1),\dotsc,(\err^*_{|L|}, \err_{|L|})\}$}
\KwOut{Linear Transform Matrix \Ai; Translation Vector \bi; Safe Radius \ri}

\prog{\Fmain}{%
\Ai, \bi $\gets$ \Fregression{$L$}\;\label{line:lin-reg}

$\ri \gets$ \Fmindist{\Ai,\bi,$\States_i$,\Inv,\Fdyn,\Fctrl,\Fmm}\;\label{line:mindist}

\Return \Ai, \bi, \ri\;
}
\caption{Construction of abstraction $\Fapprox$ for  partition $\States_i$. The output set is resented by a center defined by transformation matrix \Ai and a  vector \bi, and a ball around the center defined by  \ri.}\label{alg:infer-safe-nbr}
\end{algorithm}

To find a candidate $\Ftol_i: \Seterr \mapsto 2^\Seterr$ for a given subset $\States_i \subseteq \States$,
we consider that, when given $\err^*$ as input, $\Ftol_i$ returns a parameterized ball defined as below:
\[
\Ftol_i(\err^*) = \{\err \mid \lVert \err - (\Ai\times\err^* + \bi) \rVert \leq \ri\}
\]
where the parameters \Ai and \bi define an affine transformation from $\err^*$ to the ball's center,
and \ri defines the radius.
Here we are using a ball defined by the $\ell_2$ norm on \Seterr. Our approach generalizes to other norms and linear coordinate transformations. We discuss other norms and their effects in Section~\ref{sec:discussion}.

We start with the input to \Fmain in Algorithm~\ref{alg:infer-safe-nbr}.
Besides the subset $\States_i\subseteq\States$, the invariant \Inv, aforementioned modules \Fdyn, \Fctrl, and \Fmm,
\Fmain also requires a \emph{training set} of pairs $(\err^*, \err)$
where the $\err^*=\Fmm(\state)$ is the ground truth,
and $\err=\Fnn\circ\Fcam(\state,\envparams)$ is the percepts obtained with the vision and DNN based perception.
These pairs can be obtained from existing labeled data for training DNNs.
A labeled data point for DNNs \Fnn is already an image $\image=\Fcam(\state,\envparams)$ sampled from \States and \Envparams
and its labeled ground truth $\err^*=\Fmm(\state)$.
In practice, the state $\state=(x, y, \theta)$ can be obtained from other more accurate sensors such as GPS to label the images.
We filter the state $\state$ with the subset $\States_i$ and obtain the ground truth $\err^*=\Fmm(\state)$.
We then simply collect the perceived $\err=\Fnn(\image)$ by applying DNN on image \image.

\Fmain first uses the training set of pairs of $(\err^*, \err)$ to learn $\Ai$ and $\bi$ using multivariate linear regression.
The next section describes how it infers a safe radius \ri around the center $\Ai\times\Fmm(\state)+\bi$ by solving a constrained optimization problem.

\subsection{Making abstractions safe and precise}
\label{subsec:safe-precise} 

The \Ai and \bi computed from multivariate linear regression in Line~\ref{line:lin-reg} of \Fmain implicitly the center $\Ai\times\Fmm(\state) + \bi$ that minimizes distance to the  training data in $\States_i$.
Next, we would like to infer a safe radius \ri around the center $\Ai\times\Fmm(\state) + \bi$.
There is a tension between safety and precision in the choice of $\ri$. On one hand, we want a larger radius \ri to cover more samples, making $\Fapprox$ more conservative approximation of $\Fnn\circ\Fcam$.
On the other hand, the neighborhood should not include any \emph{unsafe perception values} that can cause a violation of the invariant \Inv.

\begin{algorithm}[h]
\DontPrintSemicolon
\KwIn{%
Linear Transform Matrix \Ai; Translation Vector \bi;\linebreak
Subspace $\States_i$; Invariant \Inv;\linebreak
Dynamics \Fdyn; Control \Fctrl; Perfect Estimation \Fmm}
\KwOut{%
Safe radius $\ri \in \mathbb{R}_{\geq0}$
}
\prog{\Fmindist}{
    solver.addVar($\state, \err,\state'$)\;
    solver.addConstraints($\state \in \States_i, \state'=\Fdyn(\state, \Fctrl(\err)), \state' \notin \Inv$)\;
    solver.setObjective($\lVert \err - (\Ai\times\Fmm(\state) + \bi) \rVert$)\;
    $\mathit{status}, \hat{r}, b$ = solver.minimize()\;\label{line:solver-outcome}
    
    \If{status = OPTIMAL $\lor$ status = SUBOPTIMAL}{
        $\ri \gets \hat{r} - b$\;\label{line:solve-opt}
    }\Else(\tcp*[h]{status = INFEASIBLE}){%
        $\ri \gets +\infty$\;
    }
    \Return \ri
}

\caption{Minimum distance to unsafe perception values.}\label{alg:min-dist}
\end{algorithm}

Formally, the set of unsafe percepts is
\(
\{\err \mid \exists \state\in\States_i. \Fdyn(\state,\Fctrl(\err))\notin\Inv\}
\)
and should be disjoint with the safe neighborhood.
Figure~\ref{fig:ex-tolerable} illustrates such a safe neighborhood for one particular state \state.
Note that $\Ftol_i$ has to extend to all states $\state\in\States_i$,
and hence we need to find a minimum \ri for any $\state\in\States_i$.
At the same time we would also like \ri as large as possible to cover more perceived values.
Further, Figure~\ref{fig:agbot5x5} shows we have to infer for all $\States_i$ in the partition.

Our solution is to find an \ri just below the minimum distance $r^*$ from the center $\Ai\times\Fmm(\state) + \bi$
to the set of unsafe percepts.
This is formalized as the constrained optimization problem below:
\begin{align*}
r^* = \min_{\state, \err,\state'} \quad & \lVert \err - (\Ai\times\Fmm(\state) + \bi) \rVert \\
                      \text{s.t.} \quad & \state \in \States_i, \state'=\Fdyn(\state, \Fctrl(\err)), \state' \notin \Inv
\end{align*}
Observe that $\state\in\States_i$ is a set of simple bounds on each state variables by designing the partition.
$\state'\notin\Inv$ is simply the invariant predicate over state variables.
However, the third constraint $\state'=\Fdyn(\state, \Fctrl(\err))$ encodes the controller \Fctrl and dynamics \Fdyn in optimization constraints.
Encoding the dynamic model \Fdyn as optimization constraints is a common technique in Model Predictive Control.
Encoding the controller \Fctrl can be achieved with a program analysis tool to convert each if-branch of control laws
into equality constraints between \err and controller output $\ctrlout=\Fctrl(\err)$.
An example template for Gurobi solver~\cite{gurobi} is shown in \Fmindist.

We argue \Fmain computes a function $\Ftol_i$ that returns a safe neighborhood for any ground truth percept $\Fmm(\state)$.
\begin{proposition}
\label{lem:safe-neighbor-set}
For each $\state \in \States_i$, $\Ftol_i(\Fmm(\state))$ computed by \Fmain is disjoint with the unsafe percepts, i.e,
\[
\forall\state\in\States_i.\Ftol_i(\Fmm(\state))\cap\{\err \mid \exists \state\in\States_i. \Fdyn(\state,\Fctrl(\err))\notin\Inv\}=\emptyset
\]
\end{proposition}

\begin{proof}
Our proof is to analyze the possible outcome status from the optimization solver,
and propagate each outcome through our functions.
At Line~\ref{line:solver-outcome}, the solver may return the following status:
\begin{enumerate}
    \item When \textit{status=OPTIMAL} or \textit{status=SUBOPTIMAL}, the solver returns a distance $\hat{r}$ and a bound $b$ such that the true minimum $r^*$ is within the bound,
    i.e., $\hat{r}\geq r^*$ and $\hat{r} - r^* < b$
    Modern solvers all provide the bound to address numerical error or sub-optimal solutions.
    Consequently, $\ri = \hat{r} - b$ at Line~\ref{line:solve-opt} ensures $\ri < r^*$, and hence the ball with the radius \ri is disjoint with the unsafe set.
    \item When \textit{status=INFEASIBLE},the constraints are unsatisfiable,
          i.e., the unsafe set $\{\err \mid \exists \state\in\States_i. \Fdyn(\state,\Fctrl(\err))\notin\Inv\}$ is proven to be $\emptyset$.
          We let $\ri=+\infty$ and thus $\Ftol_i$ is equivalent to the whole space of percepts $\Seterr$.
\end{enumerate}
\end{proof}

\subsection{Verifying  with abstraction: Theory \& code}
\label{sec:soundness}
In this subsection, we  summarize the claim that the abstraction $\Fapprox$ computed by \Fmain indeed assures safety of the abstract system $\ApproxSys(M)$ and show how it can be used for code-level verification. 
At a mathematical-level, the safety of $\Fapprox$ follows essentially from the construction in \Fmain.
Using Proposition~\ref{lem:safe-neighbor-set}, we can show that the constructed abstraction \Fapprox preserves the invariant \Inv.
\begin{proposition}
\label{lem:safe-neighbor-lem}
    If every function $\Ftol_i:\Seterr\rightarrow 2^\Seterr$ returns the safe neighborhood of $\States_i$ for all $i$,
    then the abstraction \Fapprox preserves the invariant \Inv.
\end{proposition}
\begin{proof}
Let us fix $\state \in \States_i$ and the corresponding ground truth percept $\Fmm(\state)$,
and $\Ftol_i(\Fmm(\state))$ represents all percepts allowed by $\Ftol_i$
Using the $\Ftol_i$ computed by \Fmain ,
we have shown in Proposition~\ref{lem:safe-neighbor-set} that $\Ftol_i(\Fmm(\state))$ does not intersect with any percept that can cause the next state $\Fdyn(\state,\Fctrl(\err))$ to leave \Inv.
We then rewrite it as, for each $\state \in \States_i$, any percept $\err\in\Ftol_i(\Fmm(\state))$ preserves the invariant \Inv, i.e,
\begin{equation}\label{eq:safe-subset}
\forall\state\in\States_i.\forall\err\in\Ftol_i(\Fmm(\state)).\Fdyn(\state,\Fctrl(\err))\in\Inv.    
\end{equation}
Therefore, the invariant \Inv is preserved for each subset $\States_i$,
The proof of Proposition~\ref{lem:safe-neighbor-lem} is then to expand Definition~\ref{def:safety-dir} with the function body of \Fapprox
and extend the guarantee from Equation~\eqref{eq:safe-subset} to all $\state\in\Inv$ simply because $\{\States_i\}_{i=1\dots N}$ covers \Inv.
\end{proof}

\begin{figure}[t]
    \centering
    \includegraphics[width=\linewidth]{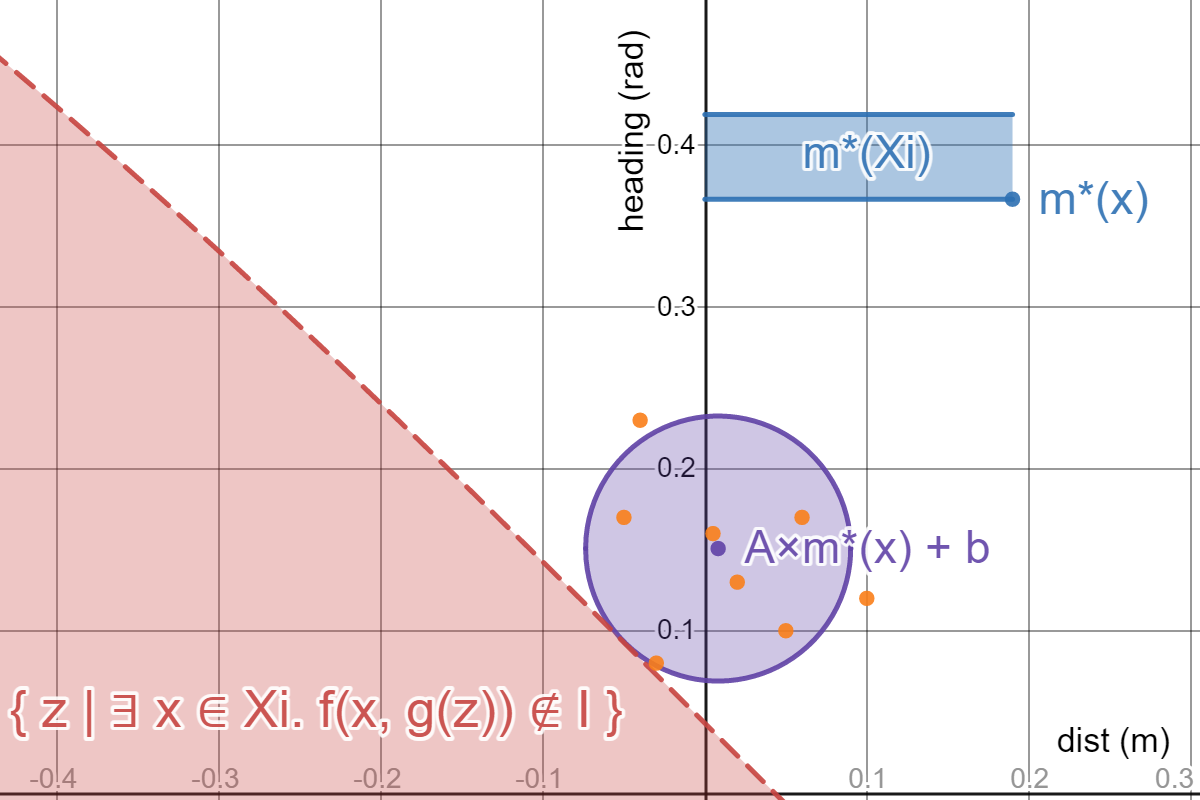}
    \caption{Example safe neighbor function $\Ftol_i$ inferred from linear regression and constrained optimization.}
    \label{fig:ex-tolerable}
\end{figure}

\begin{figure}[t]
    \centering
    \includegraphics[width=\linewidth]{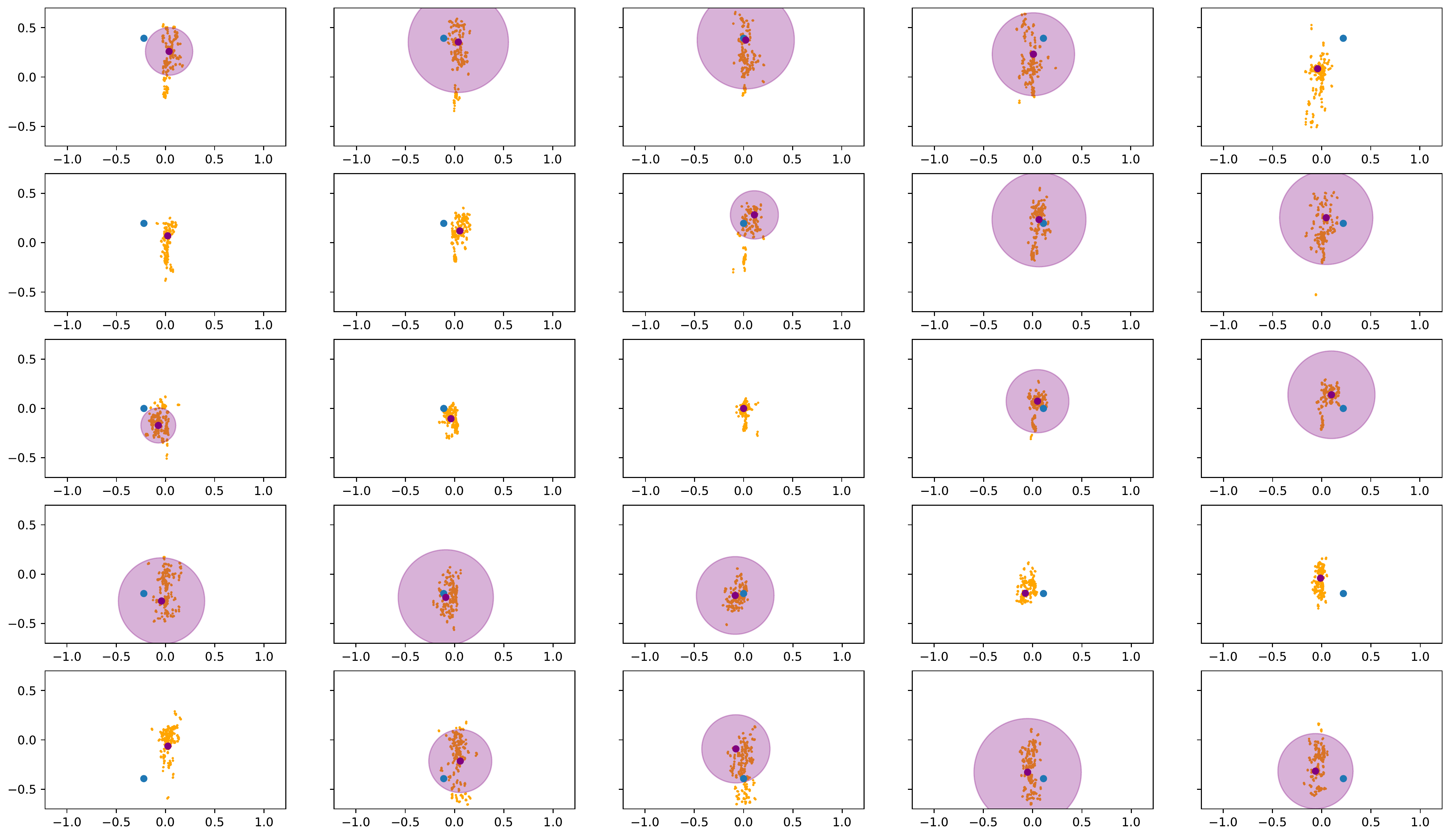}
    \caption{Ground truth (blue dot), perceived values (orange points), and inferred safe neighborhood (purple circle). Notice the biases in different parts of the state space: the mean of the perceived values do not align with the ground truth.}
    \label{fig:agbot5x5}
\end{figure}

More importantly, the constructed abstraction $\Fapprox$ can be plugged into the models of the system $\Sys$, with different levels of detail, and verified using any number of powerful formal verification tools that have been developed over the past decades. 
For example, the abstract perception system could be plugged into the  controller $\Fctrl$ and dynamics $\Fdyn$ functions represented by complex, explicit models, code, and differential equations, and we can verify the resulting system rigorously. 

To illustrate this point, in this paper, we showcase how to use $\Fapprox$ with C code implementations of $\Fctrl$ and $\Fdyn$ and verify the resulting system with CBMC~\cite{CBMC} to gain a high-level of assurance for the control system. 
Recall our set valued abstraction is of the form:
\[
\Fapprox(\state,\envparams) = \begin{cases}
\Ftol_1(\Fmm(\state)), &\text{iff} \ \state\in\States_1 \\
\qquad \vdots \\
\Ftol_N(\Fmm(\state)), &\text{iff}\ \state\in\States_N
\end{cases}
\]
This can be directly translated into program \emph{contracts}, that is, preconditions and postconditions, supported by numerous existing program analysis tools~\cite{CBMC,codeContract,javaml,specsharp}.
For instance, we are able to implement the abstraction \Fapprox shown as C code in the following template with CBMC's APIs.
\begin{lstlisting}[language=C,mathescape=true,basicstyle=\footnotesize]
$\Seterr$ $\Fapprox$($\States$ $\state$, $\Envparams$ $\envparams$){
    __CPROVER_requires($\bigvee\limits_{i=1}^N\state\in\States_i$);
    $\Seterr$ $\err$ = nondet_$\err$();
    __CPROVER_ensures($\bigwedge\limits_{i=1}^N\state\in\States_i\rightarrow\err\in\Ftol_i(\Fmm(\state))$);
}
$\Setctrlout$ $\Fctrl$($\Seterr$ $\err$) { // Stanley controller example code
  $\Setctrlout$ $\delta$ = $\err.\psi$ + atan2($K$*$\err.d$, $\vf$);
  if($\delta$ >= $\delta_{max}$)
     $\delta$ = $\delta_{max}$;
  else if($\delta$ <= $-\delta_{max}$)
     $\delta$ = $-\delta_{max}$;
  return $\delta$;
}
$\States$ $\Fdyn$($\States$ $\state$, $\Setctrlout$ $\delta$) { // Bicycle model example code
  $\States$ new_$\state$;
  new_$\state.x$ = $\state.x$ + $\vf$*cos($\state.\theta$+$\delta$)*$\dT$;
  new_$\state.y$ = $\state.y$ + $\vf$*sin($\state.\theta$+$\delta$)*$\dT$;
  new_$\state.\theta$ = $\state.\theta$ + $\vf$*sin($\delta$)/$L$*$\dT$;
  return new_$\state$;
}
\end{lstlisting}
We then are able to verify the whole system integrating the controller and the dynamics such as the example code above with CBMC.
A prominent example in the above code is that, in C, $\arctan$ with division in its input expression is often implemented with \texttt{atan2} instead to handle zero denominators correctly,
and it becomes obscure if the proof for mathematical models will still hold in the code level,
It is therefore crucial to use CBMC to automatically check safety requirements still holds.

\subsection{Precision of abstraction}
\label{subsec:precision}

How close is the computed abstraction $\Fapprox$ to the actual perception system $\Fcam \circ \Fnn$?
As we discussed earlier, it is difficult if not impossible to rigorously answer this question because the perception system (and therefore the learning stage of $\Fapprox$) depends on the $\envparams$ in complex and unknown ways. 
Also a simple answer to this question is unlikely to be satisfactory. We might care  more about precise in certain parts of the state space and for certain environmental conditions, than others.
Trying to argue which environmental conditions are more likely to arise in the real world, may be complicated. 

We propose a simple and fine-grained empirical measure of precision. 
We fix a range of environmental parameter values $\Envparams$. 
For each partition $\States_i$, we collect a \emph{testing set} of pairs of $(\err^*, \err)$ by sampling across $\States_i \times \Envparams$ using {\em some distribution $\mathcal{D}$}, where as before $\err=\Fnn\circ\Fcam(\state,\envparams)$ is the actual perception output
and $\err^* = \Fmm(\state,\envparams)$ is the ground truth. 
We denote a pair $(\err^*, \err)$ that satisfies $\err \in \Ftol_i(\err^*)$ as a positive pair.
Then, the fraction of positive pairs gives us the empirical probability with respect to $\mathcal{D}$ that  the actual  perception system (with DNN) outputs percepts that are proved to be safe in $\ApproxSys(\Fapprox)$ with respect to the invariant. 
It may be tempting to interpret this probability as a probability of system-level safety, but without additional information how $\mathcal{D}$ is related to the actual  distributions over $\States$ and $\Envparams$, we cannot make such conclusions. 

In experiments discussed in the following sections, we use $\mathcal{D}$ to be the uniform distribution. Each of the heatmaps shown in Figures~\ref{fig:heatmap-gem-stanley} illustrate the precision of  different safe abstractions over $\States_i$. A darker green $\States_i$ means that a higher fraction of outputs from the perception system matches the provably safe abstraction $\Fapprox$.

\section{Case study~1: Vision-based Lane Keeping with LaneNet}\label{sec:case-lanenet}

Recall our motivating example in Section~\ref{sec:ex1},
we study the Polaris GEM e2 Electric Vehicle and its high-fidelity Gazebo simulation~\cite{du_safe_pedestrian_2020}.
The perception module uses LaneNet~\cite{nevan_lanenet_2018} for lane detection.%
\footnote{We use \url{https://github.com/MaybeShewill-CV/lanenet-lane-detection}, one of the most popular open source implementation of LaneNet on GitHub.}
In this section, we first discuss the construction of the safe abstraction \Fapprox in Section~\ref{subsec:gem-synthesis}.
In Section~\ref{subsec:gem-heatmap}, we discuss the interpretation of the precision  \emph{heatmaps}. 
We aim study the impact of the following three factors on the precision of abstractions:
\begin{enumerate}[RQ 1.]
\item Selection of partitions $\{\States_i\}$. \label{rq:part}
\item Environment parameter distributions $\mathcal{D}$. \label{rq:env}
\item Abstractions for different invariant $\Inv$ requirements.\label{rq:prop}
\end{enumerate}

\subsection{Details on construction of abstraction}\label{subsec:gem-synthesis}
\begin{table}[t]
\caption{Definitions of tracking error functions.}\label{table:lyapunov}
\begin{tabular}{l|l}
    Tracking error function                                            & Description \\ \hline
    \multirow{2}{*}{$\lyapunov_1(d, \psi) =
                   \lVert \psi + \arctan(\frac{K\cdot d}{\vf})\rVert$} & Combined heading and \\
                                                                       & distance error in~\cite{hoffmann_stanley_2007} \\
    $\lyapunov_2(d, \psi) = \lVert d \rVert $                          & Distance error only \\
    $\lyapunov_3(d, \psi) = \lVert (d, \psi) \rVert $                  & Vector norm as error \\
\end{tabular}
\end{table}

$\States_0 = \left\{(x,y,\theta) \mid x=0 \land |y| \leq 1.2 \land |\theta| \leq \frac{\pi}{12}\right\}$ is the initial set of states,
and we recall the unsafe set is  $\Unsafe = \{(x,y,\theta) \mid |y| > 2.0\}$.
Next, we discuss about the invariant \Inv we will use to prove $\Unsafe$.
A standard induction-based proof for control systems is to define an error function (Lyapunov function) over the perceived values,
and then prove that the error is non-increasing by induction.
Formally, a tracking error function is
$\lyapunov: \Seterr \mapsto \mathbb{R}_{\geq 0}$, with $\lyapunov(\mathbf{0}_\err) = 0$ and $\lyapunov(\err) > 0$ when $\state \neq \mathbf{0}_\err$.
where $\mathbf{0}_\err \in \Seterr$ is the equilibrium.
The different  error functions used in this paper are shown in Table~\ref{table:lyapunov}.

Take $\lyapunov_1$ in Table~\ref{table:lyapunov} as an example,
the ideal perceptual output of a state $(x,y,\theta)$ is obtained by $(d^*, \psi^*) = \Fmm(x, y, \theta)$, and the tracking error is then $\lyapunov_1(d^*, \psi^*)$.
This function $\Fmm$ may in general be complicated and dependent on the geometry of the lanes.
The next state according to Equation~\eqref{eq:closed-system}
is $(x',y',\theta')=\Fdyn((x,y,\theta),\Fctrl(\Fmm(x,y,\theta)))$.
The next percept is then obtained by $({d^*}', {\psi^*}') = \Fmm(x', y', \theta')$.
We then define the invariant of non-increasing error $\Inv_1\subseteq\States$ as:
\[
\Inv_1 = \{\state \mid \lyapunov_1({d^*}', {\psi^*}') \leq \lyapunov_1(d^*, \psi^*)\}
\]
We give the detail descriptions and values of constant symbols used in \Fdyn and \Fctrl as well as the perfect estimation $\Fmm$ in Appendix~\ref{appx:stanley}.

To infer the safety abstraction $\Fapprox$, we consider the invariant is covered by partitions
$\{\States_i\}_{i\leq N}$ with $y$ within $\pm 1.2$ meters and heading angle $\theta$ within $\pm 15^\circ$,
that is,
\[
\bigcup_{i=1}^N\States_i = \left\{(x,y,\theta) \mid |y|\leq 1.2 \land |\theta|\leq\frac{\pi}{12}\right\}
\]
Further, we consider three different partitions $N\in\{8\times5, 8\times10, 8\times20\}$; larger numbers partition more finely and produce refinements of the coarser abstractions.
Here we do not partition along $x$ for (a) better visualization and (b) because lanes are aligned with the $x$-axis,
partitioning $x$-axis does not produce interesting results.

To prepare the training data for learning \Ai and \bi to construct $\Ftol_i$,
we use the Gazebo simulator in~\cite{du_safe_pedestrian_2020} to generate camera images \image
labeled with their ground truth percepts $\err^*$.
Each image is sampled from an uniform distribution $\mathcal{D}$ over the subspace $\States_i$
as well as the environment space \Envparams.
The environment parameter space \Envparams is defined by:
\begin{inparaenum}[(i)]
\item three types of straight roads, two-lane, four-lane, and six-lane,
\item two different lighting conditions, day and dawn.
\end{inparaenum}
The ground truth percept $\err^*=\Fmm(\state)$ is then calculated using information from simulator,
and we ensure that at least 300 images are collected for each $\States_i$ to learn \Ai and \bi for $\Ftol_i$.

For each partition, given \Ai and \bi learned from multivariate linear regression using the data.
\Fmindist, implemented in Gurobi~\cite{gurobi}, solves the following 
 nonlinear optimization problem to find \ri:
Each $\States_i$ covers an interval of 0.3 meter for $y$ and $3^\circ$ for $\theta$.
    We discuss the optimization problem for a particular subset $\States_i$ that covers $y$ from 0.9 to 1.2 meters and $\theta$ from $12^\circ$ to $15^\circ$ as an example,   i.e, $\States_i=\left\{(x,y,z)\mid y\in[0.9,1.2] \land \theta\in\left[\frac{\pi}{15}, \frac{\pi}{12}\right]\right\}$.
    
    \begin{align*}
    \min \quad & \lVert (d,\psi) - (\Ai\times\Fmm(x,y,\theta) + \bi) \rVert \\
    s.t. \quad & y \in [0.9,1.2], \theta\in\left[\frac{\pi}{15}, \frac{\pi}{12}\right] \\
    & \delta = \Fctrl(d,\psi) \\
    & x' = x + \vf\cdot\cos(\theta + \delta)\cdot\dT \\
    & y' = y + \vf\cdot\sin(\theta + \delta)\cdot\dT \\
    & \theta' = \theta + \vf\cdot\frac{\sin(\delta)}{L}\cdot\dT \\
    & \lyapunov_1(\Fmm(x',y',\theta')) > \lyapunov_1(\Fmm(x,y,\theta))
    \end{align*}
All the computed abstractions were composed with the code for the controller $\Fctrl$ and the dynamics $\Fdyn$ and verified for the corresponding invariant with CBMC.

\subsection{Interpretation of precision of abstractions}
\label{subsec:gem-heatmap}

Figure~\ref{fig:heatmap-gem-stanley} shows the precision maps for   three  abstractions resulting from three increasingly finer partitions and two  sets of testing environments. First, we discuss the broad trends and then delve into the details. 

\paragraph{At equilibrium, abstraction breaks but it does not matter}
All six heatmaps demonstrate a common trend where there is  a band of white (low score) cells going from the second to the fourth quadrant.
There are areas where the safe radius $\ri$ of $\Ftol_i$ is too small be an abstraction of the perception system. 
This phenomenon can be informally understood as follows:
First, the center (equilibrium) of the plot corresponds to near zero error in deviation $d$ and heading $\psi$. 
Consider when a vehicle state with the tracking error approach 0,
the percept must also approach the ground truth so that the next state would maintain the error does not increase.
Consequently, the safe radius $\ri\to 0$.
Recall that \ri is minimized with respect to all state $\state\in\Seterr_i$.
We can view the overall system $\ApproxSys(\Fapprox)$ as a fixed-resolution quantized control system. It is well-known that such a system cannot achieve perfect asymptotic stability~\cite{brockett2000quantized}. The feedback does not have enough resolution to drive the state to the equilibrium, but instead, it can converge to some neighborhood around it. This explains why the error function $V$ cannot be non-increasing around the origin. In other words, here the abstraction is ``failing to be safe'' because of the control. We note that not being able to prove safety around the origin is less of a problem because these are precisely the states where the vehicle is centered between the lanes and its heading is aligned.

\paragraph{Weak invariants can break safe abstraction}
Secondly, along the diagonal line (through the origin) we have states where the vehicle's deviation from the lane center $d$ and the heading $\psi$ are in opposing direction. 
By observing $\lyapunov_1$ from Table~\ref{table:lyapunov},
we know $\psi$ and $d$ are of opposite signs at the equilibrium points $\lyapunov_1(d, \psi)$.
Hence the band of white cells goes from the second to the fourth quadrant.
Therefore, the tracking error cannot be non-increasing in these states in one step as required by $\Inv$. These regions of the precision map are white because, as mentioned, the safe radius $\ri$ is too small to be an abstraction of $\Ftol$. In these regions, the abstraction is failing because the invariant $\Inv$ property we are trying to prove is too weak to be proven in one step. A remedy for this problem will be to come up with stronger inductive invariants for the system with perfect perception.  

\begin{figure}[t]
\setkeys{Gin}{width=0.32\linewidth,clip,trim={2cm 0cm 2.2cm 1cm}}
\includegraphics{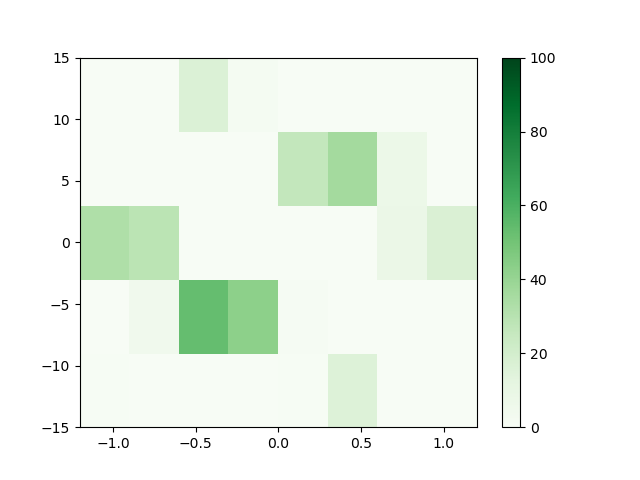}
\includegraphics{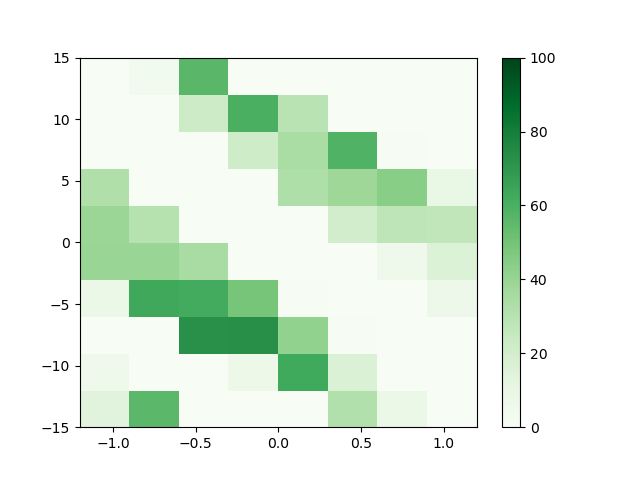}
\includegraphics{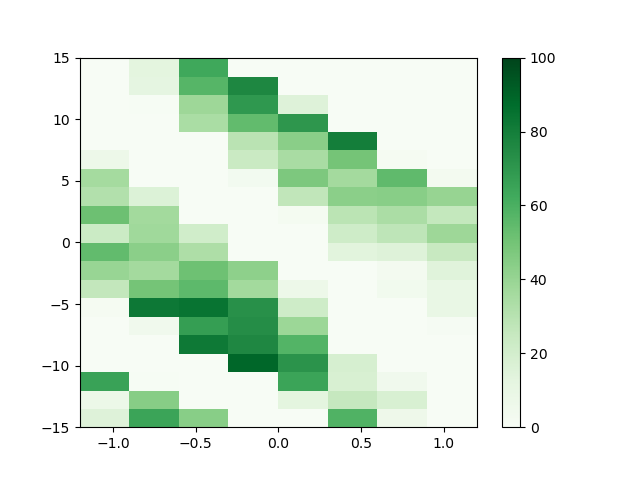}

\includegraphics{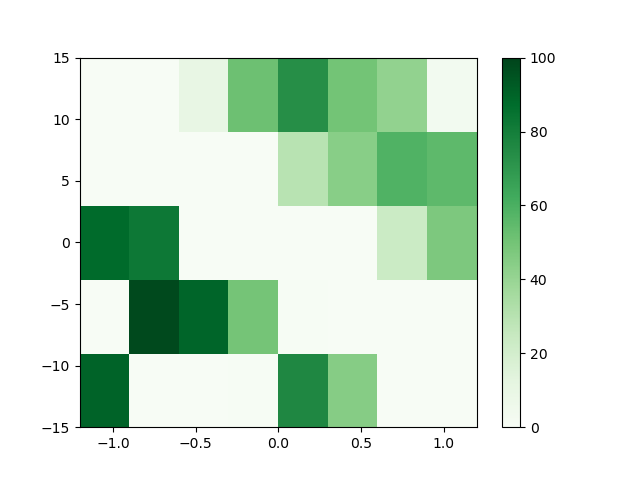}
\includegraphics{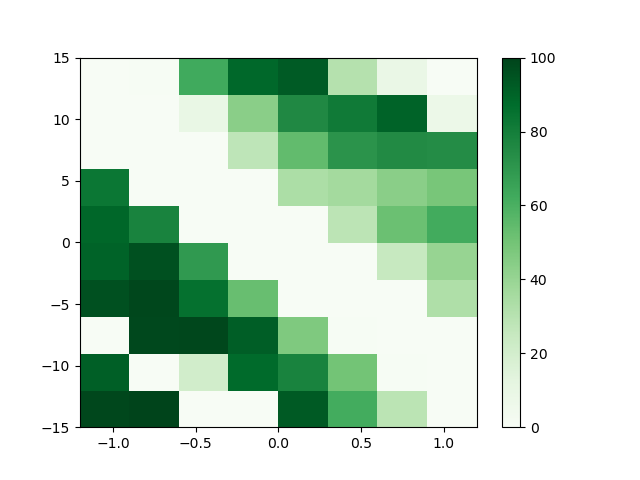}
\includegraphics{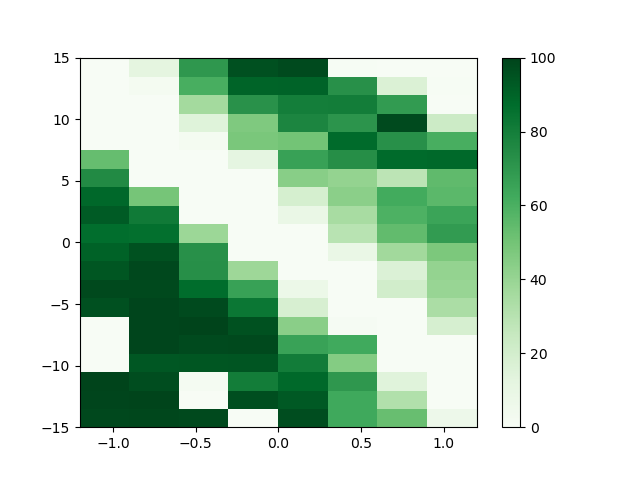}
\caption{Precision heatmaps of abstractions inferred for LaneNet with Stanley controller.
    The partitions with $N=8\times5$~(\emph{Left column}), $N=8\times10$~(\emph{Mid column}), and $N=8\times20$~(\emph{Right column}).
    The environment parameter space with three road types~(\emph{Top row}) and only two-lane road (\emph{Bottom row}).
}\label{fig:heatmap-gem-stanley}
\end{figure}

\paragraph{Finer partitions improve precision of safe abstractions}
We observe from each row of heatmaps in Figure~\ref{fig:heatmap-gem-stanley} that finer partitions generate more precise abstractions.
In the finest partition, several cells achieve over 90 percent.
The reasons are twofold.
\begin{inparaenum}
\item With a finer partition, linear regression can better fit a smaller interval of the original nonlinear perception.
\item The safe radius \ri is minimized for all states $\state\in\States_i$.
      If a smaller subset $\States_j\subset\States_i$ excludes the worst state,
      the radius $r_j$ can improve and cover larger safe neighborhood.
\end{inparaenum}

\paragraph{Fewer environmental variations improve precision}
For RQ~\ref{rq:env}, we observe each column of heatmaps in Figure~\ref{fig:heatmap-gem-stanley}.
We generated two testing sets under different distributions over the environment space including
 \begin{inparaenum}
 \item the same uniform distribution for the training set, and
 \item an uniform distribution over the subspace with only the two-lane road.
\end{inparaenum}
The colors become darker for the same cell locations as expected.
The variance in the perceived values by DNN reduces because of the fewer environmental variations.
The same radius now can cover more samples in the testing set.

\paragraph{Variations with different safety requirements}
Finally for RQ~\ref{rq:prop}, we consider another invariant which uses different tracking error functions $\lyapunov_2$ and $\lyapunov_3$ (listed in Table~\ref{table:lyapunov}).
$\lyapunov_2$ considers only lane deviation error ($d$), and $\lyapunov_3$ uses the vector norm as error.
Both can be used to prove the same unsafe set \Unsafe.
Three heatmaps for each tracking error function are shown in Figure~\ref{fig:heatmap-gem-cte} for the same three partitions and with the testing set with two-lane road.
By comparing Figures~\ref{fig:heatmap-gem-stanley} and~\ref{fig:heatmap-gem-cte},
we see a white band now surrounding the line $d=0$ for $\lyapunov_2$ and a white spot around the origin for $\lyapunov_3$.
This validates our explanation that the abstraction breaks owing to the stringent requirement of non-increasing error.

\begin{figure}[t]
\setkeys{Gin}{width=0.32\linewidth,clip,trim={2cm 0cm 2.2cm 1cm}}
\includegraphics{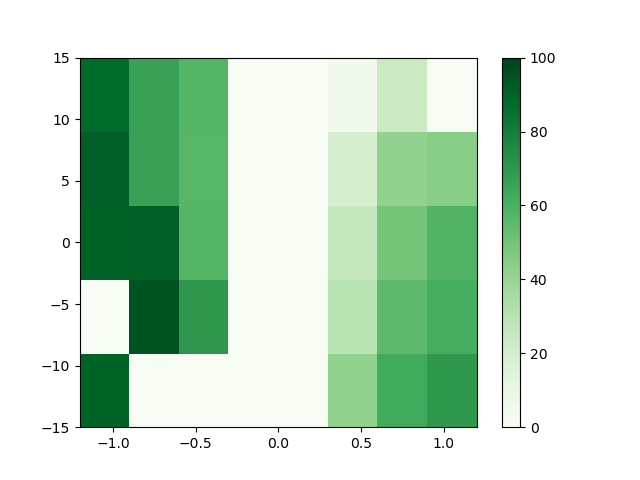}
\includegraphics{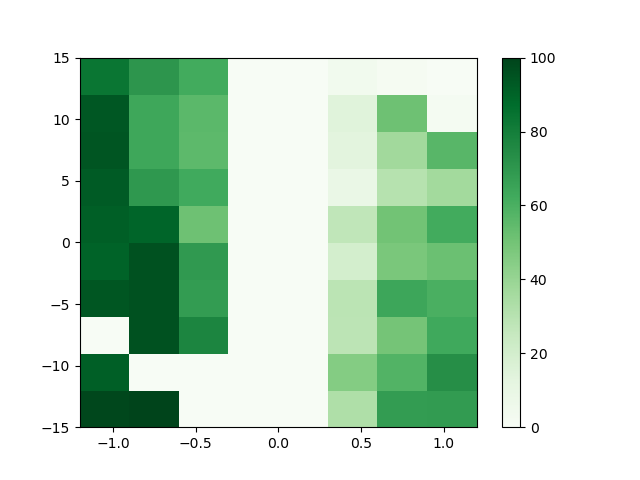}
\includegraphics{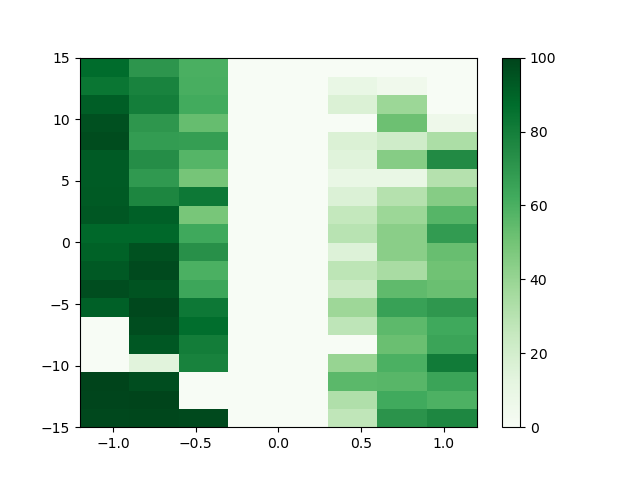}

\includegraphics{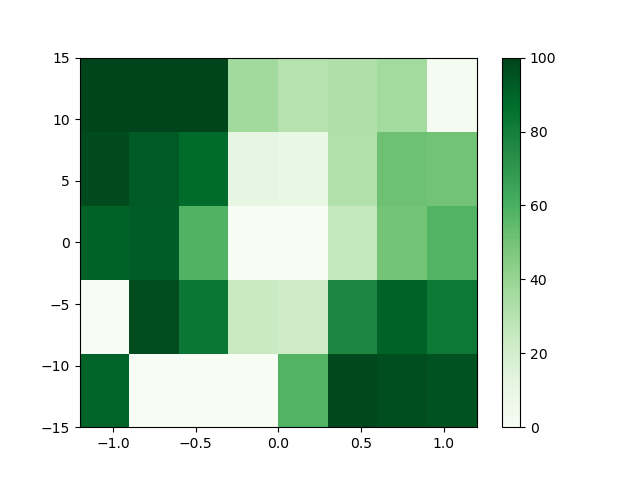}
\includegraphics{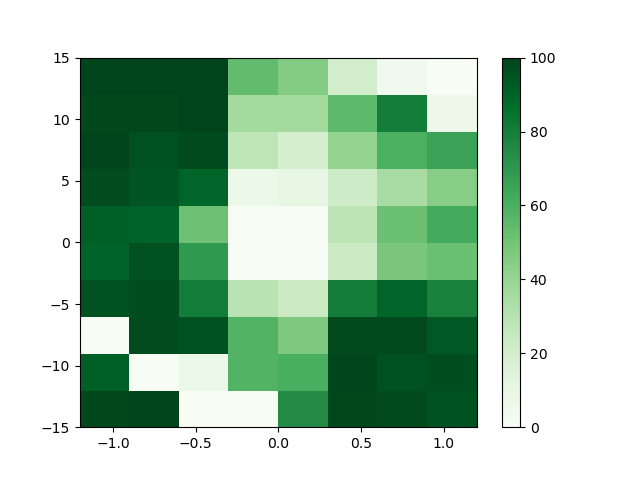}
\includegraphics{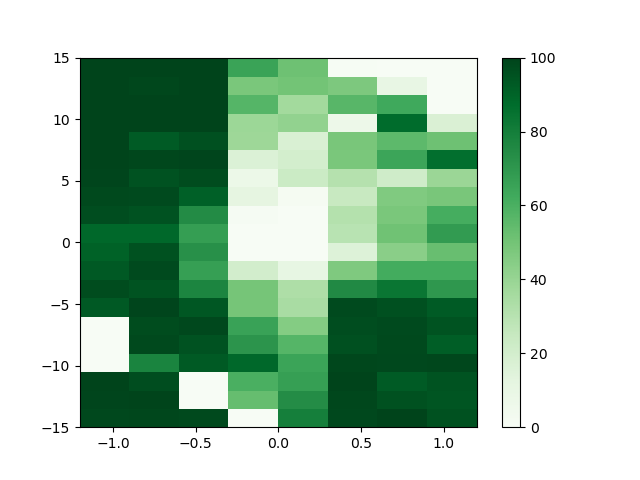}
\caption{Precision heatmaps for LaneNet with Stanley controller with only two-lane road and two error tracking functions $\lyapunov_2$~(\emph{Top Row}) vs $\lyapunov_3$~(\emph{Bottom Row}).}\label{fig:heatmap-gem-cte}
\end{figure}

\section{Case Study~2: Corn Row Following Agbot}\label{sec:case-agbot}

\begin{figure}[ht]
    \centering
    \includegraphics[width=0.45\linewidth]{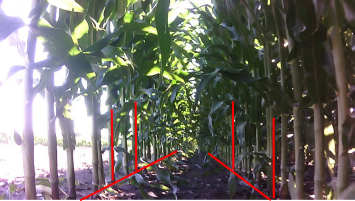}
    \includegraphics[width=0.45\linewidth,clip,trim= 0 100 0 20]{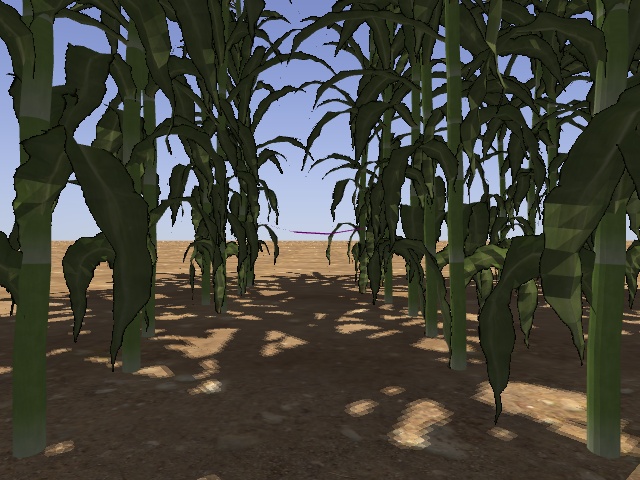}
    \caption{Real and simulated camera images used in corn row following for agricultural robots.}
    \label{fig:agbot-cam}
\end{figure}

Our second case study is the visual navigation system of under-canopy agricultural robot (AgBot), CropFollow, developed in~\cite{sivakumar_visual_agbot_2021}.
The system is responsible for the lateral control when the vehicle traverses the space between two rows of crops.
Similar to our first case study, the system captures the image in front of the vehicle with a camera (Figure~\ref{fig:agbot-cam}),
applies a ResNet-18 CNN on the camera image to perceive the relative positions of the corn rows to the ego vehicle,
and uses a modified Stanley controller to reduce the lateral deviation.

Here we give the detailed definition of each component.
In CropFollow~\cite{sivakumar_visual_agbot_2021}, the vehicle dynamics is approximated with a kinematic differential model of a skid-steering mobile robot.
The state \state consists of the 2D position $x$ and $y$ and the heading $\theta$.
The input $\ctrlout$ is the desired angular velocity $\omega$.
The dynamics $\Fdyn(\state, \ctrlout)$ is:
\begin{align*}
x_{t+1} &= x_t + \vf\cdot\cos(\theta_t)\cdot\dT \\
y_{t+1} &= y_t + \vf\cdot\sin(\theta_t)\cdot\dT \\
\theta_{t+1} &= \theta_t + \omega\cdot\dT
\end{align*}

Likewise, the modified Stanley controller takes a percept $\err \in \Seterr$ composed of the heading difference $\psi$ and cross track distance $d$
to an imaginary center line of two corn rows,
and outputs the angular velocity $\omega$.
The controller $\Fctrl$ is given as:
\[
\Fctrl(d, \psi) = \begin{cases}
    \frac{\psi + \arctan\left(\frac{K\cdot d}{\vf}\right)}{\dT}, & \text{if } \left|\psi + \arctan\left(\frac{K\cdot d}{\vf}\right)\right| < \omega_{max}\cdot\dT  \\
    \omega_{max},                         & \text{if } \psi + \arctan\left(\frac{K\cdot d}{\vf}\right) \geq \omega_{max}\cdot\dT  \\
    -\omega_{max},                        & \text{if } \psi + \arctan\left(\frac{K\cdot d}{\vf}\right) \leq -\omega_{max}\cdot\dT
\end{cases}
\]
For the farm robots, we wish to avoid two undesirable outcomes:
\begin{inparaenum}
\item if $|y| > 0.228$ meters, the vehicle will hit the corn, and
\item if $|\theta| > 30^\circ$, the neural network output becomes highly inaccurate and recovery may be impossible.
\end{inparaenum}
Therefore, we define the unsafe set $\Unsafe=\{(x,y,\theta) \mid |y| > 0.228 \land |\theta| > \frac{\pi}{6}\}$.
We use the same tracking error function $\lyapunov_1$ to specify the invariant \Inv with the above defined \Fdyn and \Fctrl.
Note that the values of constant symbols are tuned differently and provided in Appendix~\ref{appx:stanley-agbot}.

Similarly, we consider three different partitions $N\in\{5\times 5,10\times 10,20\times20\}$ to cover the invariant;
the whole space $\bigcup_{i=1}^N\States_i$ covers $\pm 0.228$ meters in $y$ and $\pm 30^\circ$ in $\theta$.
We follow the same procedure to sample images and derive the safe neighbor function $\Ftol_i$ for $\States_i$.
For this case study,
the environment parameter space \Envparams is defined by five different plant fields, including three stages of corn (baby, small, and adult) and two stages of tobacco (early and late) fields.
We use the uniform distribution over the state space $\States_i$ and the five environment parameters for both the training testing set. 

Figure~\ref{fig:heatmap-agbot} shows the precision heatmaps for the abstractions inferred with three different partitions.
We observe almost identical broad trends compared to Figure~\ref{fig:heatmap-gem-stanley},
including the white band around equilibrium,
the white spots in the upper right and lower left corners close to the violation of invariant,
and higher precision score with finer partitions.
This case study reaffirms the validity of our interpretation over the precision heatmap in Section~\ref{sec:case-lanenet}.
It also showcases that our analysis can be applied on vastly different vision and DNN based perception system with similar percept space.

\begin{figure}
\setkeys{Gin}{width=0.32\linewidth,clip,trim={2cm 0cm 2.2cm 1cm}}
\includegraphics{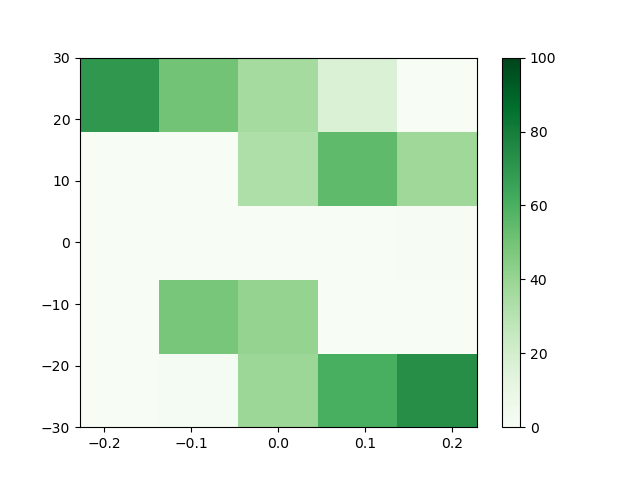}
\includegraphics{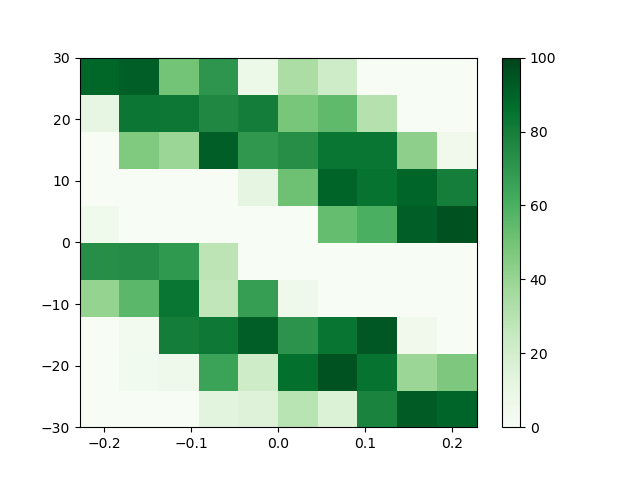}
\includegraphics{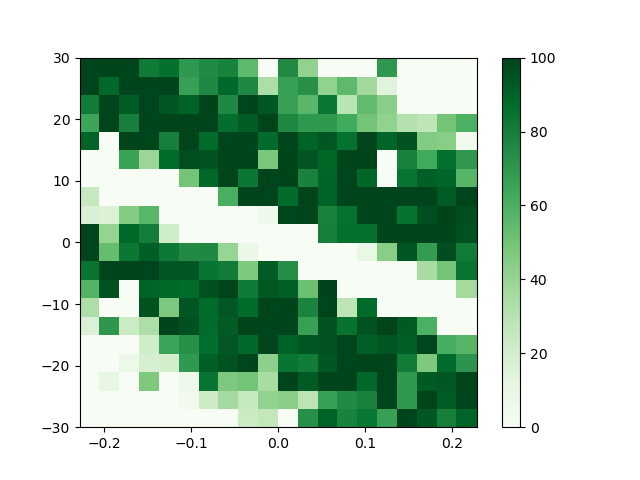}
\caption{Precision heatmap of abstractions inferred for CropFollow with modified Stanley controller.
The partitions with $N=5\times5$~(\emph{Left}), $N=10\times10$~(\emph{Mid}), and $N=20\times20$~(\emph{Right})}\label{fig:heatmap-agbot}
\end{figure}

\section{Discussion and Future Directions}
\label{sec:discussion}

Safety assurance of autonomous systems that use machine learning models for perception is an important challenge. We presented an approach for constructing abstractions for verifying control systems that use convolutional neural networks for perception. The approach learns  piece-wise affine set-valued abstractions of the perception system from data. It maximizes the these sets for improving precision, while assuring a given safety requirement. Viewing abstractions of perception systems along the triple axes of safety, intelligibility, and precision may be a productive perspective for tackling the problems of safety assurance of autonomous systems. We discuss some of the lessons learned and the future research directions they suggest.

Within the space of intelligible abstractions, we have explored one corner with piece-wise affine models. Needless to say that this was a somewhat arbitrary choice, and many other options should be explored, for example, with decision trees, polynomial models, space partitioning algorithms, etc.
Developing algorithms for computing such abstractions and as well as verifying the end-to-end abstract system $\ApproxSys(M)$ would be interesting directions for future research. 

Our piece-wise affine abstractions used uniform rectangular partitions. We observed that the size of the partitions have significant impact on the precision of the safe abstractions. The results suggest that non-uniform or adaptive partitioning (e.g., finer partitions nearer to the equilibrium) would yield more precise abstractions. Using domain knowledge and symmetries in creating the abstractions should  substantially improve their precision and  size. 

As expected, the safety requirement (or invariant) guiding the  construction of the abstraction significantly impacts the precision of the abstraction. The precision maps shed light on parts of the state space and environment where the DNN-based perception system is most fragile, likely to violate requirements. Such quantitative insights can inform design decisions for the perception system, the control system, as well as the definition of the system-level {\em operating design domains (ODDs)}.

Finally, we chose to use discrete time models and used CBMC for verifying the closed system with the abstraction. Extending the approach to continuous and hybrid would be interesting and would require nontrivial extensions of existing verification tools.



\bibliographystyle{ACM-Reference-Format}
\bibliography{references,xai}

\appendix

\section{Stanley Controller with GEM car}\label{appx:stanley}

\begin{table}[H]
    \caption{Description and values of constants for Polaris GEM e2 Electric Cart model}
    \begin{tabular}{l|r|l}
        Symbol         & Value & Description                     \\ \hline
        $\vf$          &   2.8 & Constant forward velocity (m/s) \\
        $L$            &  1.75 & Wheel base (m)                  \\
        $\dT$          &   0.1 & Time discretization (s)         \\
        $\delta_{max}$ &  0.61 & Steering angle limit (rad)      \\
        $K$            &  0.45 & Stanley controller gain
    \end{tabular}
\end{table}

\paragraph{Non-increasing cross track distance}

Following the proof in~\cite{hoffmann_stanley_2007},
when $|\psi + \arctan(\frac{K\cdot d}{\vf})| < \delta_{max}$,
\begin{align*}
   \dot{d} &= -\vf\cdot\sin(\arctan(\frac{K\cdot d}{\vf})) \\
           &= -\frac{K\cdot d}{\sqrt{1 + (\frac{K\cdot d}{\vf})^2}}
\end{align*}

Note that $\lVert d\rVert$ converges to zero because $-\frac{K\cdot d}{\sqrt{1 + (\frac{K\cdot d}{\vf})^2}}$
is always the opposite sign of $d$.
We can find the Lyapunov function for nominal region $\lyapunov_2(d, \psi) = \lVert d \rVert$.
This is however not entirely true in discrete dynamics because the value $\lVert d \rVert$ can cross zero and become larger in a discrete transition.

\paragraph{Non-increasing vector norm value}
Following the proof in~\cite{hoffmann_stanley_2007},
when $|\psi + \arctan(\frac{K\cdot d}{\vf})| < \delta_{max}$,
\begin{align*}
\dot{d}    &= -\frac{K\cdot d}{\sqrt{1 + (\frac{K\cdot d}{\vf})^2}} \\
\dot{\psi} &= -\frac{\vf\cdot\sin(\psi + \arctan(\frac{K\cdot d}{\vf}))}{L}
\end{align*}

The sign of $\dot{\psi}$ is opposite that of $(\psi + \arctan(\frac{K\cdot d}{\vf}))$,
so $\psi$ approaches $\arctan(\frac{K\cdot d}{\vf})$.
Further, $\arctan(\frac{K\cdot d}{\vf})$ converges to zero because $d$ converges to zero as proven above.
Therefore, the origin is the only equilibrium, and the 2D vector norm is non-increasing.

\section{Modified Stanley Controller with farm robots}\label{appx:stanley-agbot}

\begin{table}[H]
    \caption{Description and values of constant Symbols for agricultural robots}
    \begin{tabular}{l|r|l}
        Symbol         & Value & Description                     \\ \hline
        $\vf$          &   1.0 & Constant forward velocity (m/s) \\
        $\dT$          &  0.05 & Time discretization (s)         \\
        $\omega_{max}$ &   0.5 & Angular velocity limit (rad/s)  \\
        $K$            &   0.1 & Stanley controller gain
    \end{tabular}
\end{table}

\end{document}